\documentclass{article}

\usepackage[final,nonatbib]{neurips_2020}

\usepackage[utf8]{inputenc} 
\usepackage[T1]{fontenc}    
\usepackage{hyperref}       
\usepackage{url}            
\usepackage{booktabs}       
\usepackage{amsfonts}       
\usepackage{nicefrac}       
\usepackage{microtype}      
\usepackage{subfigure}
\usepackage{algorithm}
\usepackage{algorithmic}
\usepackage{amsmath,amsfonts,amssymb,amsthm}
\usepackage{thmtools,thm-restate}
\usepackage{todonotes}
\usepackage{tabularx}
\usepackage{enumitem}
\usepackage{caption}

\newtheorem{thm}{Theorem}

\newtheorem{lem}[thm]{Lemma}

\theoremstyle{definition}
\newtheorem{defn}{Definition}

\theoremstyle{remark}

\newcommand{\myparagraph}[1]{\noindent\textbf{#1} }

\def\shownotes{1}  
\ifnum\shownotes=1
\newcommand{\authnote}[2]{{$\ll$\textsf{\footnotesize #1 notes: #2}$\gg$}}
\else
\newcommand{\authnote}[2]{}
\fi 
\newcommand{\YL}[1]{}
\newcommand{\sid}[1]{}

\title{Functional Regularization for Representation Learning: A Unified Theoretical Perspective}

\author{Siddhant Garg\thanks{\ \ Work completed at the University of Wisconsin-Madison}\\
Amazon Alexa AI Search\\
Manhattan Beach, CA, USA\\
\texttt{sidgarg@amazon.com} \\
\And Yingyu Liang\\
Department of Computer Sciences\\
University of Wisconsin-Madison\\
\texttt{yliang@cs.wisc.edu}
}

\begin{document}

\maketitle

\begin{abstract}
Unsupervised and self-supervised learning approaches have become a crucial tool to learn representations for downstream prediction tasks. While these approaches are widely used in practice and achieve impressive empirical gains, their theoretical understanding largely lags behind. Towards bridging this gap, we present a unifying perspective where several such approaches can be viewed as imposing a regularization on the representation via a \emph{learnable} function using unlabeled data. We propose a discriminative theoretical framework for analyzing the sample complexity of these approaches, which generalizes the framework of~\cite{10.1145/1706591.1706599} to allow learnable regularization functions. Our sample complexity bounds show that, with carefully chosen hypothesis classes to exploit the structure in the data, these learnable regularization functions can prune the hypothesis space, and help reduce the amount of labeled data needed. We then provide two concrete examples of functional regularization, one using auto-encoders and the other using masked self-supervision, and apply our framework to quantify the reduction in the sample complexity bound of labeled data. We also provide complementary empirical results to support our analysis. 
\end{abstract}

\section{Introduction}

Advancements in machine learning have resulted in large prediction models, which need large amounts of labeled data for effective learning. 
Expensive label annotation costs  have increased the popularity of unsupervised (or self-supervised) representation learning techniques using additional unlabeled data. 
These techniques learn a representation function on the input, and a prediction function over the representation for the target prediction task.
Unlabeled data is utilised by posing an auxiliary unsupervised learning task on the representation, e.g., using the representation to reconstruct the input.
Some popular examples of the auxiliary task are auto-encoders~\cite{6302929,10.5555/3045796.3045801}, sparse dictionaries~\cite{6626561}, 
masked self-supervision~\cite{devlin-etal-2019-bert}, manifold learning~\cite{cayton2005algorithms}, among others~\cite{bengio2012representation}. 
These approaches have been extensively used in applications in domains such as computer vision (e.g.,~\cite{10.1145/1390156.1390294,48bcd1ffcba540238cc38891dd919d4d,NIPS2014_5548}) and natural language processing (e.g.,~\cite{pmlr-v70-yang17d,devlin-etal-2019-bert,DBLP:journals/corr/abs-1907-11692}), and have achieved impressive empirical performance. 

An important takeaway from these empirical studies is that learning representations using unlabeled data can drastically reduce the size of labeled data needed for the prediction task.  
In contrast to the popularity and impressive practical gains of these representation learning approaches, there have been far fewer theoretical studies focused at understanding them, most of which have been specific to individual approaches. 
While intuition dictates that the unlabeled and labeled data distributions along with the choice of models are crucial factors which govern the empirical gains, theoretically there is still ambiguity over questions like \textit{"When can the auxiliary task over the unlabeled data help the target prediction task? How much can it reduce the sample size of the labeled data by?"}

In this paper, we take a step towards improving the theoretical understanding of the benefits of learning representations for the target prediction task via an auxiliary task. We focus on analyzing the sample complexity of labeled and unlabeled data for this learning paradigm. Such an analysis can help to identify conditions when a significant reduction in sample complexity of the labeled data can be achieved. Arguably, this is one of the most fundamental questions for this learning paradigm, and existing literature on this has been limited and scattered, specific to individual approaches.

Our contribution is to propose a unified perspective where several representation learning approaches can be viewed as if they impose a regularization on the representation via a learnable regularization function. Under this paradigm, representations are learned jointly on unlabeled and labeled data. The former is used in the auxiliary task to learn the representation and the regularization function. The latter is used in the target prediction task to learn the representation and the prediction function. Henceforth, we refer to this paradigm as \textit{representation learning via functional regularization}. 

In particular, we present a PAC-style discriminative framework~\cite{valiant1984theory} to bound the sample complexities of labeled and unlabeled data under different assumptions on the models and data distributions. This is inspired from the work of~\cite{10.1145/1706591.1706599} which bounds the sample complexities of labeled and unlabeled data for semi-supervised learning. Our generalized framework allows \emph{learnable} regularization functions and thus unifies multiple unsupervised (or self-supervised) representation learning approaches. Our analysis shows that functional regularization with unlabeled data can prune the model hypothesis class for learning representations, reducing the labeled data required for the prediction task. 

To demonstrate the application of our framework, we construct two concrete examples of functional regularization, one using auto-encoder and the other using masked self-supervision. These specific functional regularization settings allow us to quantify the reduction in the sample bounds of labeled data more explicitly. While our main focus is the theoretical framework, we also provide complementary empirical support through experiments on synthetic and real data. Now we first discuss related work followed by a formal problem description. Then we present our theoretical framework involving sample complexity bounds followed by the concrete examples with empirical support.

\section{Related Work}
\label{sec:related_work}

Self-supervised learning approaches for images have been extensively used in computer vision through auxiliary tasks such as masked image patch prediction~\cite{NIPS2014_5548}, image rotations~\cite{gidaris2018unsupervised}, pixel colorization~\cite{inproceedingszhang}, context prediction of image patches~\cite{10.1109/ICCV.2015.167,Pathak2016ContextEF,noroozi2016,hnaff2019dataefficient}, etc. 
Additionally, variants of these approaches find practical use in the field of robotics~\cite{article_sofman,Pinto2015SupersizingSL,article_pulkit,Ebert2017SelfSupervisedVP,pmlr-v87-jang18a}. Masked self-supervision (a type of denoising auto-encoder), where representations are learnt by hiding a portion of the input and then reconstructing it, has lead to powerful language models like BERT~\cite{devlin-etal-2019-bert} and RoBERTa~\cite{DBLP:journals/corr/abs-1907-11692} in natural language processing. 
There have also been numerous studies on other representation learning approaches such as
RBMs~\cite{10.1145/1390156.1390224,10.5555/2503308.2188407}, 
dictionary learning~\cite{rodriguez2008sparse,NIPS2008_3448} and manifold learning~\cite{NIPS2011_4409}; \cite{bengio2012representation} presents an extensive review of multiple representation learning approaches.
 
On the theoretical front,~\cite{10.1145/1706591.1706599} presents a discriminative framework for analyzing semi-supervised learning showing that unlabeled data can reduce the labeled sample complexity.
Our framework in this paper is inspired from~\cite{10.1145/1706591.1706599}, and generalizes their analysis for utilizing unlabeled data through a \emph{learnable} regularization function. This allows a unified theoretical framework to study multiple representation learning approaches. In addition to~\cite{10.1145/1706591.1706599}, \cite{10.5555/1208768} also studies the benefits of using unlabeled data, but by restricting that the unlabeled data be utilized through a fixed function.
Some other works~\cite{ben2010theory,10.1007/978-3-642-34106-9_14} have explored the benefits of unlabeled data for domain adaptation. Our setting differs from this since our goal is to learn a prediction function on the labeled data, rather than for a change in the domain of labeled data from source to target.
Another line of related work considers multi-task learning, such as~\cite{10.5555/1622248.1622254, 10.1109/TPAMI.2016.2544314}. These works show that multiple supervised learning tasks on different, but related, data distributions can help generalization. Our work differs from these since we focus on learning a supervised task using auxiliary unsupervised tasks on unlabeled data.

\cite{10.5555/1756006.1756025} presents a comprehensive empirical study on the benefits of unsupervised pre-training for image-classification tasks in computer vision. Our analysis in this paper is motivated by their empirical results showing that pre-training shrinks the hypothesis space searched during learning.
There have also been theoretical studies on several representation approaches individually, without providing a holistic perspective. \cite{pmlr-v97-saunshi19a} presents a theoretical framework to analyse unsupervised representation learning techniques that can be posed as a contrastive learning problem, 
with their results later improved by~\cite{nozawa2019pacbayesian}. \cite{10.5555/3157382.3157468} provide a theoretical analysis of unsupervised learning from an optimization viewpoint, with applications to dictionary learning and spectral auto-encoders. 
\cite{NIPS2018_7296} prove uniform stability generalization bounds for linear auto-encoders and empirically demonstrate the benefits of using supervised auto-encoders. Additionally, there are some studies on learning transferable representations using multiple tasks~\cite{du2020few,tripuraneni2020theory,hanneke2020no}. Another line of related work includes approaches~\cite{bachman2019learning,tschannen2020mutual} that analyze representation learning from the perspective of maximizing the mutual information between the data and the representation. Connecting these mutual information approaches with our framework is left as future work.

\section{Problem Formulation}
\label{sec:problem_defn}
Consider labeled data $S = \{(x_i, y_i)\}_{i=1}^{m_\ell}$ from a distribution $\mathcal{D}$ over the domains $\mathcal{X} \times \mathcal{Y}$, where $\mathcal{X} \subseteq \mathbb{R}^d$ is the input feature space and $\mathcal{Y}$ is the label space. The goal is to learn a predictor $p:\mathcal{X} \rightarrow \mathcal{Y}$ that fits $\mathcal{D}$. This can be achieved by first learning a representation function $\phi = h(x) \in \mathbb{R}^r$ over the input, and then learning a predictor $y {=} f(\phi) {\in} \mathcal{Y}$ on the representation. Denote the hypothesis classes for $h$ and $f$ by $\mathcal{H}$ and $\mathcal{F}$ respectively, and the loss function by $\ell_c(f{(h(x))},y)$. Without loss of generality, we assume $\ell_c \in [0,1]$. 
We are interested in representation learning approaches where $ h(x)$ is learned with the help of an auxiliary task on unlabeled data $U=\{\Tilde{x}_i\}_{i=1}^{m_u}$ from a distribution $\mathcal{U}_X$ (same or different from the marginal distribution $\mathcal{D}_X$ of $\mathcal{D}$). Such an auxiliary task is posed as an unsupervised (or rather a self-supervised) learning task depending only on the input feature $x$. 

Representation learning using auto-encoders is an example that fits this consideration, where given input $x$, the goal is to learn $\phi=h(x)$ s.t. $x$ can be decoded back from $h(x)$. More precisely, the decoder $d$ takes the representation $\phi$ and decodes it to $\hat{x} {=} g(\phi) \in \mathbb{R}^d$. $h$ and $g$ are learnt by minimizing the reconstruction error between $\hat{x}$ and $x$ (e.g., $\|x - \hat{x}\|_2 {=} \|x {-} g(h(x))\|_2$). Representation learning via masked self-supervision is another example of our problem setting, where the goal is to learn $\phi=h(x)$ such that a part of the input (e.g., $x_1$) can be predicted from the representation of the remaining part (e.g., $h(x')$ on input $x' = [0, x_2, \ldots, x_d]$ with $x_1$ masked). This approach uses a function $g$ over  the representation to predict the masked part of the input. $h$ and $g$ are optimized by minimizing the error between the true and the predicted values (e.g., $(x_1 {-} g(h(x')))^2$).

Now consider a simple example of a regression problem where $y{=} \sum_{i=1}^d x_i$ and we use masked self-supervision to learn $x_1$ from $x'$. If each $x_i {\sim} \mathcal{N}(0, 1)\  i.i.d.$, then $h$ will not be able to learn a meaningful representation of $x$ for predicting $y$, since $x_1$ is independent of all other coordinates of $x$. On the other extreme, if all $x_i$'s are equal, $h$ can learn the perfect underlying representation for predicting $y$, which corresponds to a single coordinate of $x$. This shows two contrasting abstractions of the inherent structure in the data and how the benefits of using a specific auxiliary task may vary. Our work aims at providing a framework for analyzing sample complexity and clarifying these subtleties on the benefits of the auxiliary task depending on the data distribution. 

\section{Functional Regularization: A Unified Perspective}
\label{sec:framework} 

We make a key observation that the auxiliary task in several representation learning approaches provides a regularization on the representation function via a learnable function. To better illustrate this viewpoint, consider the auxiliary task of an auto-encoder, where the decoder $g(\phi)$ can be viewed as such a learnable function, and the reconstruction error $\|x - g(h(x))\|_2$ can be viewed as a regularization penalty imposed on $h$ through the decoder $g$ for the data point $x$. 

To formalize this notion, we consider learning representations via an auxiliary task which involves: a learnable function $g$, and a loss of the form $L_r(h, g ; x)$ on the representation $h$ via $g$ for an input $x$. We refer to $g$ as the regularization function and $L_r$ as the regularization loss. Let $\mathcal{G}$ denote the hypothesis class for $g$. Without loss of generality we assume that $L_r \in [0,1]$.  

\begin{defn}
Given a loss function $L_r(h, g; x)$ for an input $x$ involving a representation $h$ and a regularization function $g$, the regularization loss of $h$ and $g$ on a distribution $\mathcal{U}_X$ over $\mathcal{X}$ is defined as
\begin{align} 
  L_r(h, g\ ; \mathcal{U}_X) := \mathbb{E}_{x \sim \mathcal{U}_X} \left[  L_r(h, g; x) \right]. 
\end{align}
The regularization loss of a representation function $h$ on $\mathcal{U}_X$ is defined as
\begin{align} 
  L_r(h\ ; \mathcal{U}_X) := \min_{g\in \mathcal{G}} L_r(h, g\ ; \mathcal{U}_X).
\end{align}
We can similarly define $L_r(h, g\ ; U)$ and $L_r(h\ ; U)$ to denote the loss over a fixed set $U$ of unlabeled data points, i.e., $L_r(h, g\ ; U) := \frac{1}{|U|}\sum_{x \in U} L_r(h, g\ ; x)$ and $L_r(h\ ; U) := \min_{g\in \mathcal{G}} L_r(h, g\ ; U)$. 
\end{defn}
Here, $L_r(h\ ;\mathcal{U}_X)$ can be viewed as a notion of incompatibility of a representation function $h$ on the data distribution $\mathcal{U}$. This formalizes the prior knowledge about the representation function and the data. For example, in auto-encoders $L_r(h\ ;\mathcal{U}_X)$ measures how well the representation function $h$ complies with the prior knowledge of the input being reconstructible from the representation. 

We now introduce a notion for the subset of representation functions having a bounded regularization loss, which is crucial for our sample complexity analysis.
\begin{defn}
Given $\tau \in [0,1]$, the $\tau$-regularization-loss subset of representation hypotheses $\mathcal{H}$ is:
\begin{align} 
    \mathcal{H}_{\mathcal{D}_X, L_r}(\tau)  := \{h \in \mathcal{H}: L_r(h\ ; \mathcal{D}_X) \le \tau\}.
\end{align}
\end{defn}

We also define the prediction loss over the data distribution $\mathcal{D}$ for a prediction function $f$ on top of $h$: 
$
  L_c(f, h\ ; \mathcal{D})  := \mathbb{E}_{(x,y) \sim \mathcal{D}} \left[  \ell_c(f(h(x)), y) \right],
$
where $\ell_c$ is the loss function for prediction defined in Section~\ref{sec:problem_defn}.
Similarly, the empirical loss on the labeled data set $S$ is $L_c(f, h\ ; S) := \frac{1}{|S|}\sum_{(x,y) \in S} \ell_c(f(h(x)), y)$. In summary, given hypothesis classes $\mathcal{H}, \mathcal{F}$, and $\mathcal{G}$, a labeled dataset $S$, an unlabeled dataset $U$, and a threshold $\tau>0$ on the regularization loss, we consider the following learning problem:
\begin{align}
\label{eqn:learning_problem}
    \min_{f \in \mathcal{F}, h \in \mathcal{H}} L_c(f, h\ ; S), \textrm{~~s.t.~~} L_r(h\ ; U) \le \tau.
\end{align}

\subsection{Instantiations of Functional Regularization}
Here we show some popular representation learning approaches as instantiations of our framework by specifying the analogous regularization functions $\mathcal{G}$ and regularization losses $L_r$. We mention several other instantiations of our framework like manifold learning, dictionary learning, etc in Appendix~\ref{app:instance}. 

\myparagraph{Auto-encoder.} Recall from Section~\ref{sec:problem_defn}, there is a decoder function $g$ that takes the representation $h(x)$ and decodes it to $\hat{x} {=} g(h(x))$, where $h$ and $g$ are learnt by minimizing the error $\|x - \hat{x}\|_2$. The reconstruction error corresponds to the regularization loss $L_r(h,g\ ;x)$. $\mathcal{H}_{\mathcal{D}_X, L_r}(\tau)$ is the subset of representation functions with at most $\tau$ reconstruction error using the best decoder in $\mathcal{G}$.

\myparagraph{Masked Self-supervised Learning.}
Recall the simple example from Section~\ref{sec:problem_defn} where $x_1$, the first dimension of input $x$, is predicted from the representation $h(x')$ of the remaining part (i.e., $x' = [0, x_2, \ldots, x_d]$ with $x_1$ masked).
The prediction function for $x_1$ using $h(x')$ corresponds to the regularization function $g$, and $\|x_1 {-} g(h(x'))\|^2_2$ corresponds to the regularization loss $L_r(h,g\ ;x)$. $\mathcal{H}_{\mathcal{D}_X, L_r}(\tau)$ is the subset of $\mathcal{H}$ which have at most $\tau$ MSE on predicting $x_1$ using the best function $g$.
More general denoising auto-encoder methods can be similarly mapped in our framework.

\myparagraph{Variational Auto-encoder.} 
VAEs encode the input $x$ as a distribution $q_{\phi}(z|x)$ over a parametric latent space $z$ instead of a single point, and sample from it to reconstruct $x$ using a decoder $p_{\theta}(x|z)$. The encoder $q_{\phi}(z|x)$ is used to model the underlying mean $\mu_z$ and co-variance matrix $\sigma_z$ of the distribution over $z$. VAEs are trained by minimising a loss over parameters $\theta$ and $\phi$
$$
\mathcal{L}_{x}(\theta,\phi) = -\mathbb{E}_{z \sim q_{\phi}(z|x)}[\log \ p_{\theta}(x|z)] + \mathbb{KL}(q_{\phi}(z|x)||p(z))
$$ 
where $p(z)$ is specified as the prior distribution over $z$ (e.g., $\mathcal{N}(0,1)$). The encoder $q_{\phi}(z|x)$ can be viewed as the representation function $h$, the decoder $p_{\theta}(x|z)$ as the learnable regularization function $g$, and the loss $\mathcal{L}_{x}(\theta,\phi)$ as the regularization loss $L_r(h, g; x)$ in our framework. Then $\mathcal{H}_{\mathcal{D}_X, L_r}(\tau)$ is the subset of encoders $q_{\phi}$ which have at most $\tau$ VAE loss when using the best decoder $p_\theta$ for it.

\subsection{Sample Complexity Analysis}\label{sec:complexity}

To analyze the sample complexity of representation learning via functional regularization, we generalize the analysis from~\cite{10.1145/1706591.1706599} by extending it from semi-supervised learning with unlabeled data using a fixed regularization function to the general setting of using \emph{learnable} regularization functions with unlabeled data.
We first enumerate the different considerations on the data distributions and the hypothesis classes: 1) the labeled and unlabeled data can either be from the same or different distributions (i.e., same domain or different domains); 2) the hypothesis classes can contain zero error hypothesis or not (i.e., being realizable or unrealizable); 3) the hypothesis classes can be finite or infinite in size. We perform sample complexity analysis for different combinations of these assumptions. While the bounds across different settings have some differences, the proofs share a common high-level intuition. We now present sample complexity bounds for three interesting, characteristic settings. We present  bounds for several other settings along with proofs of all the theorems in Appendix~\ref{app:bound}.

\myparagraph{Same Domain, Realizable, Finite Hypothesis Classes.}
We start with the simplest setting, where the unlabeled dataset $U$ and the labeled dataset $S$ are from the same distribution $\mathcal{D}_X$, and the hypothesis classes $\mathcal{F}, \mathcal{G}, \mathcal{H}$ contain functions $ f^*, g^*, h^*$ with a zero prediction and regularization loss. We further assume that the hypothesis classes are finite in size. We can prove the following result:
\begin{restatable}{thm}{samerealizable}
\label{thm:samerealizable}
Suppose there exist $h^* \in \mathcal{H}, f^* \in \mathcal{F}, g^* \in \mathcal{G}$ such that  $L_c(f^*, h^*; \mathcal{D})=0 $ and $L_r(h^*, g^*; \mathcal{D}_X) = 0$.
For any $\epsilon_0, \epsilon_1 \in (0, 1/2)$, a set $U$ of $m_u$ unlabeled examples and a set $S$ of $m_l$ labeled examples are sufficient to learn to an error $\epsilon_1$ with probability $1 - \delta$, where
\begin{align}
\label{eqn:thm1_bounds}
    \begin{split}
    m_u \ge \frac{1}{\epsilon_0} \left[ \ln|\mathcal{G}| + \ln| \mathcal{H}| + \ln\frac{2}{\delta} \right],
    \end{split}
    \begin{split}
    m_l \ge \frac{1}{\epsilon_1} \left[ \ln|\mathcal{F} | + \ln |\mathcal{H}_{\mathcal{D}_X, L_r}(\epsilon_0) | + \ln\frac{2}{\delta} \right].
    \end{split}
\end{align}
In particular, with probability at least $1 - \delta$, all hypotheses $h \in \mathcal{H}, f \in \mathcal{F}$ with $L_c(f, h; S) = 0$ and
$L_r(h; U) = 0$ will have $L_c(f, h; \mathcal{D}) \le \epsilon_1$.
\end{restatable}
  
Theorem~\ref{thm:samerealizable} shows that, if the target function $f^* {\circ} h^*$ is indeed perfectly correct for the target prediction task, and $h^* {\circ} g^*$ has a zero regularization loss, then optimizing the prediction and regularization loss to 0 over the above mentioned number of data points (Equation~\ref{eqn:thm1_bounds}) is sufficient to learn an accurate predictor in the PAC learning sense. 

Recall that PAC analysis for the standard setting of the prediction task only using labeled data (and no unlabeled data) shows that $\frac{1}{\epsilon_1} \left[ \ln|\mathcal{F} | + \ln|\mathcal{H} | + \ln\frac{2}{\delta} \right] $ labeled points are needed to get the same error guarantee. On comparing the bounds, Theorem~\ref{thm:samerealizable} shows that the functional regularization can prune away some hypotheses in $\mathcal{H}$; thereby replacing the factor $\mathcal{H}$ with its subset $\mathcal{H}_{\mathcal{D}_X, L_r}(\epsilon_0)$ in the bound. Thus, the sample complexity bound is reduced by $\frac{1}{\epsilon_1} \left[ \ln|  \mathcal{H}| {-} \ln |  \mathcal{H}_{\mathcal{D}_X, L_r}(\epsilon_0)| \right ]$. Equivalently, the error is reduced by $\frac{1}{m_\ell} \left[ \ln| \mathcal{H} | {-} \ln |  \mathcal{H}_{\mathcal{D}_X, L_r}(\epsilon_0) | \right ]$ when using $m_\ell$ labeled data.
So the auxiliary task is helpful for learning the predictor when $\mathcal{H}_{\mathcal{D}_X, L_r}(\epsilon_0)$ is significantly smaller than $\mathcal{H}$, avoiding the requirement of a large number of labeled points to find a good representation function among them. 

\myparagraph{Same Domain, Unrealizable, Infinite Hypothesis Classes.}
We now present the result for a more elaborate setting, where both the prediction and regularization losses are non-zero. We also relax the assumptions on the hypothesis classes being finite. We use metric entropy to measure the capacity of the hypothesis classes for demonstration here. Alternative capacity measures like VC-dimension or Rademacher complexity can also be used with essentially no change to the analysis.
Assume that the parameter space of $\mathcal{H}$ is equipped with a norm and let $\mathcal{N}_\mathcal{H}(\epsilon)$ denote the $\epsilon$-covering number of $\mathcal{H}$; similarly for $\mathcal{F}$ and $\mathcal{G}$. Let the Lipschitz constant of the losses w.r.t.\ these norms be bounded by $L$.

\begin{restatable}{thm}{thmcovering}
\label{thm:covering}
Suppose there exist $h^* \in \mathcal{H}, f^* \in \mathcal{F}, g^* \in \mathcal{G}$ such that  $L_c(f^*, h^*; \mathcal{D}) \le \epsilon_c $ and $L_r(h^*, g^*; \mathcal{D}_X) \le \epsilon_r$.
For any $\epsilon_0, \epsilon_1 \in (0, 1/2)$, a set $U$ of $m_u$ unlabeled examples and a set $S$ of $m_l$ labeled examples are sufficient to learn to an error $\epsilon_c + \epsilon_1$ with probability $1 - \delta$, where
\begin{align}
    m_u & \ge \frac{C}{\epsilon^2_0} \ln\frac{1}{\delta} \left[ \ln\mathcal{N}_{\mathcal{G}}\left(\frac{\epsilon_0}{4L}\right) {+} \ln\mathcal{N}_{\mathcal{H} }\left(\frac{\epsilon_0}{4L}\right) \right],
    \\
    m_l & \ge \frac{C}{\epsilon^2_1} \ln\frac{1}{\delta} \left[ \ln\mathcal{N}_{\mathcal{F} }\left(\frac{\epsilon_1}{4L}\right) + \ln\mathcal{N}_{ \mathcal{H}_{\mathcal{D}_X, L_r}(\epsilon_r {+} 2\epsilon_0) }\left(\frac{\epsilon_1}{4L}\right) \right]  \label{eqn:bound_unrealizable}
\end{align}
for some absolute constant $C$.
In particular, with probability at least $1 - \delta$, the $h \in \mathcal{H}, f \in \mathcal{F}$ that optimize $L_c(f, h; S)$ subject to $L_r(h; U) \le \epsilon_r + \epsilon_0$ have $L_c(f, h; \mathcal{D}) \le L_c(f^*, h^*; \mathcal{D})  + \epsilon_1$.
\end{restatable}

Theorem~\ref{thm:covering} shows that optimizing the prediction loss subject to the regularization loss bounded by $(\epsilon_r + \epsilon_0)$ can give a solution $f \circ h$ with prediction loss close to the optimal. The sample complexity bounds are broadly similar to those in the simpler realizable and finite hypothesis class setting, but with $\frac{1}{\epsilon_0}$ and $\frac{1}{\epsilon_1}$ replaced by $\frac{1}{\epsilon_0^2}$ and $\frac{1}{\epsilon_1^2}$ due to the unrealizability, and logarithms of the hypothesis class sizes replaced by their metric entropy due to the classes being infinite. 
We show a reduction of $\frac{C}{\epsilon_1^2} \left[ \ln\mathcal{N}_{\mathcal{H} }\left(\frac{\epsilon_1}{4L}\right) {-} \ln\mathcal{N}_{ \mathcal{H}_{\mathcal{D}_X, L_r}(\epsilon_r {+} 2\epsilon_0) }\left(\frac{\epsilon_1}{4L}\right)\right ]$ with the standard bound on $m_l$ without unlabeled data. 
Equivalently, the error bound is reduced by $\frac{C}{\sqrt{m_\ell}} \left[ \ln\mathcal{N}_{\mathcal{H} }\left(\frac{\epsilon_1}{4L}\right) {-}  \ln\mathcal{N}_{ \mathcal{H}_{\mathcal{D}_X, L_r}(\epsilon_r {+} 2\epsilon_0) }\left(\frac{\epsilon_1}{4L}\right) \right ]$.

We bring attention to some subtleties which are worth noting. Firstly, the regularization loss $\epsilon_r$ of $g^*, h^*$ need not be optimal; there may be other $g, h$ which get a smaller $L_r(h, g; \mathcal{D}_X)$ (even $\ll \epsilon_r$). Secondly, the prediction loss is bounded by $L_c(f^*, h^*; \mathcal{D}) + \epsilon_1$, which is independent of $\epsilon_r$. Similarly, the bounds on $m_u$ and $m_\ell$ mainly depend on $\epsilon_0$ and $\epsilon_1$ respectively, while only $m_\ell$ depends on $\epsilon_r$ through the $\mathcal{H}_{\mathcal{D}_X, L_r}(\epsilon_r {+} 2\epsilon_0)$ term. Thus, even when the regularization loss is large (e.g., the reconstruction of an auto-encoder is far from accurate), it is still possible to learn an accurate predictor with a significantly reduced labeled data size using the unlabeled data. This suggests that when designing an auxiliary task ($\mathcal{G}$ and $L_r$), it is \emph{not} necessary to ensure that the ``ground-truth'' $h^*$ has a small regularization loss. Rather, one should ensure that only a small fraction of $h \in \mathcal{H}$ have a smaller (or similar) regularization loss than $h^*$ so as to reduce the label sample complexity. 

This bound also shows that $\tau$ should be carefully chosen for the constraint $L_r(h; U) \le \tau$. With a very small $\tau$, the ground-truth $h^*$ (or hypotheses of similar quality) may not satisfy the constraint and become infeasible for learning. With a very large $\tau$, the auxiliary task may not reduce the labeled sample complexity. Practical learning algorithms typically turn this constrain into a regularization like term, i.e., by optimizing $L_c(f, h; S) + \lambda L_r(h; U)$. For such objectives, the requirement on $\tau$ translates to carefully choosing $\lambda$. When $\lambda$ is very large, this leads to a small $L_r(h; U)$ but a large $L_c(f, h; S)$, while when $\lambda$ is very small, this may not reduce  the labeled sample complexity.

\myparagraph{Different Domain, Unrealizable, Infinite Hypothesis Classes.}
In practice, the unlabeled data is often from a different domain than the labeled data. For example, state-of-the-art NLP systems are often pre-trained on a large general-purpose unlabeled corpus (e.g., the entire Wikipedia) while the target task has a small specific labeled corpus (e.g., a set of medical records). That is, the unlabeled data $U$ is from a distribution $\mathcal{U}_X$ different from $\mathcal{D}_X$. For this setting, we have the following bound: 

\begin{restatable}{thm}{diffdomain}
\label{thm:diff_domain}
Suppose the unlabeled data $U$ is from a distribution $\mathcal{U}_X$ different from $\mathcal{D}_X$. 
Suppose there exist $h^* \in \mathcal{H}, f^* \in \mathcal{F}, g^* \in \mathcal{G}$ such that  $L_c(f^*, h^*; \mathcal{D}) \le \epsilon_c$ and $L_r(h^*, g^*; \mathcal{U}_X) \le \epsilon_r$. Then the same sample complexity bounds as in Theorem~\ref{thm:covering} hold (replacing $\mathcal{D}_X$ with $\mathcal{U}_X$ in Equation~\ref{eqn:bound_unrealizable}).
\end{restatable}

The bound is similar to that in the setting of the domain distributions being same. 
It implies that unlabeled data from a domain different than the labeled data, can help in learning the target task, as long as there exists a ``ground-truth'' representation function $h^*$, which is shared across the two domains, having a small prediction loss on the labeled data and a suitable regularization loss on the unlabeled data. The former (small prediction loss) is typically assumed according to domain knowledge, e.g., 
for image data, common visual perception features are believed to be shared across different types of images. The latter (suitable regularization loss) means only a small fraction of $h {\in} \mathcal{H}$ have a smaller (or similar) regularization loss than $h^*$, which requires a careful design of $\mathcal{G}$ and $L_r$.

\subsection{Discussions}

\paragraph{When is functional regularization not helpful?}
In addition to demonstrating the benefits of unlabeled data for the target prediction task, our theorems and analysis also provide implications for cases when the auxiliary self-supervised task may \emph{not} help the target prediction task. 
Firstly, the regularization may not lead to a significant reduction in the size of the hypothesis class. For example, consider Theorem 1, if $\mathcal{H}_{\mathcal{D}_X, L_r}(\epsilon_0)$ is not significantly smaller than $\mathcal{H}$, then using unlabeled data will not reduce the sample size of the labeled data by much when compared to the case of only using labeled data for prediction. In fact, to get significant gain in sample complexity reduction, the size of the regularized hypothesis class $\mathcal{H}_{\mathcal{D}_X, L_r}(\epsilon_0)$ needs to be exponentially smaller than entire class $\mathcal{H}$. A polynomially smaller $\mathcal{H}_{\mathcal{D}_X, L_r}(\epsilon_0)$ only leads to minor logarithmic reduction in the sample complexity. Section~\ref{sec:example} presents two concrete examples where the regularized hypothesis class is exponentially smaller than $\mathcal{H}$ thereby showing benefits of using functional regularization, but this can also help to identify examples where the auxiliary task does not aid learning. 
Secondly, the auxiliary task can fail if the regularization loss threshold ($\tau$ in Equation~(\ref{eqn:learning_problem})) is not properly set. For example, if $\tau$ is set too small, then the feasible set ($\mathcal{H}_{\mathcal{D}_X, L_r}(\tau)$) may contain no hypotheses with a small prediction loss.
Lastly, another possible reason that these representation learning approaches may fail is the inability of the optimization to lead to a good solution. Analyzing the effects of optimization for function regularization is an interesting future direction. 

\paragraph{Is uniform convergence suitable for our analysis?}
Our sample complexity analysis is based on uniform convergence bounds. A careful reader may question whether uniform convergence is suitable for analyzing the generalization in the first place, since there is evidence~\cite{zhang2016understanding,nagarajan2019uniform} that na\"ively applying uniform convergence bounds may not result in good generalization/sample bounds for deep learning. 
However, these existing studies~\cite{zhang2016understanding,nagarajan2019uniform} cannot be directly applied to our problem setting. To the best of our knowledge, they apply for supervised learning tasks without any auxiliary representation learning on unlabeled data, which differs from our setting of using auxiliary tasks on unlabeled data. 
This difference in problem settings is the key in making uniform convergence bounds meaningful. More precisely, in supervised deep learning without auxiliary tasks, it is generally believed that the hypothesis class is larger than statistically necessary, and the optimization has an implicit regularization while training, and hence uniform convergence fails to explain the generalization (e.g.,~\cite{neyshabur2014search,li2018algorithmic}). However, in our setting with the auxiliary tasks, functional regularization has a regularization effect of restricting the learning to a smaller subset of the hypothesis space, as shown by our analysis and supported by empirical evidence in existing works (e.g.,~\cite{10.5555/1756006.1756025}) and our experiments in Section~\ref{sec:exp} and Appendix~\ref{app:additional_expts}. Once regularized to a smaller subset of hypotheses, the implicit regularization of the optimization is no longer significant, and thus the generalization can be explained by uniform convergence. A more thorough investigation is left as future work.

\section{Applying the Theoretical Framework to Concrete Examples} \label{sec:example}

The analysis in Section~\ref{sec:framework} shows that the sample complexity bound reduction depends on the notion of the pruned subset $\mathcal{H}_{\mathcal{D}_X, L_r}$, which captures the effect of the regularization function and the property of the unlabeled data distribution. 
Our generic framework can be applied to various concrete configurations of hypothesis classes and data distributions. This way we can quantify the reduction more explicitly by investigating $\mathcal{H}_{\mathcal{D}_X, L_r}$. We provide two such examples: one using an auto-encoder regularization and the other using a masked self-supervision regularization.
We outline how to bound the sample complexity for these examples, and present the complete details and proofs in Appendix~\ref{app:example}.

\subsection{An Example of Functional Regularization via Auto-encoder}
\label{subsec:vae_example}

\myparagraph{Learning Without Functional Regularization.}
Consider $\mathcal{H}$ to be the class of linear functions from $\mathbb{R}^d$ to $\mathbb{R}^r$ where ${r < d/2}$, and $\mathcal{F}$ to be the class of linear functions over some activations. That is,
\begin{align}
\label{eqn:auto-encoder}
\phi = h_W(x) = W x, ~~y = f_a(\phi) = \sum_{i=1}^r a_i \sigma(\phi_i) \ , \textrm{~where~} W \in \mathbb{R}^{r \times d}, ~~ a \in \mathbb{R}^{r}
\end{align}
\setlength{\belowdisplayskip}{3pt}
Here $\sigma(t)$ is an activation function (e.g., $\sigma(t) {=} t^2$), the rows of $W$ and $a$ have $\ell_2$ norms bounded by 1. We consider the Mean Square Error prediction loss, i.e., $L_c(f, h; x) {=} \|y {-} f(h(x))\|^2_2$. Without prior knowledge on data, no functional regularization corresponds to end-to-end training on $\mathcal{F} {\circ} \mathcal{H}$. 

\myparagraph{Data Property.}
We consider a setting where the data has properties which allows functional regularization. 
We assume that the data consists of a signal and noise, where the signal lies in a certain $r$-dimensional subspace. Formally, let columns of $B \in \mathbb{R}^{d\times d}$ be eigenvectors of $\Sigma {:=}\mathbb{E}[xx^\top]$, then the prediction labels are largely determined by the signal in the first $r$ directions: $y {=}  \sum_{i=1}^r a^*_i \sigma(\phi^*_i) {+} \nu$
and $\phi^* {=} B_{1:r}^\top x$, where $a^* \in \mathbb{R}^r$ is a ground-truth parameter with $\|a^*\|_2 {\le} 1$, $B_{1:r}$ is the set of first $r$ eigenvectors of $\Sigma$, and $\nu$ is a small Gaussian noise. We assume a difference in the $r^{\text{th}}$ and ${r{+}1}^{\text{th}}$ eigenvalues of $\Sigma$ to distinguish the corresponding eigenvectors.
Let $\epsilon_r$ denote $\mathbb{E}\|x {-} B_{1:r} B_{1:r}^\top x  \|_2^2$. 

\myparagraph{Learning With Functional Regularization.}
Knowing that the signal lies in an $r$-dimensional subspace, we can perform auto-encoder functional regularization.
Let $\mathcal{G}$ be a class of linear functions from $\mathbb{R}^r$ to $\mathbb{R}^d$, i.e., $\hat{x} {=} g_V(\phi) {=} V\phi$ where $V {\in} \mathcal{R}^{d \times r}$ has orthonormal columns. The regularization loss 
$L_r(h, g; x) {=} \|x {-} g(h(x))\|^2_2$. For simplicity, we assume access to infinite unlabeled data.

Without regularization, the standard $\epsilon$-covering argument shows that the labeled sample complexity, for an error $ \epsilon \ $ close to the optimal, is $\frac{C}{\epsilon^2} \left[ \ln \mathcal{N}_{\mathcal{F}}\left(\frac{\epsilon}{4L}\right)  + \ln \mathcal{N}_{\mathcal{H}}\left(\frac{\epsilon}{4L}\right) \right]$ for some absolute constant $C$. 
Applying our framework when using regularization with $\tau=\epsilon_r$, the sample complexity is bounded by $\frac{C}{\epsilon^2} \left[ \ln \mathcal{N}_{\mathcal{F}}\left(\frac{\epsilon}{4L}\right) + \ln\mathcal{N}_{ \mathcal{H}_{\mathcal{D}_X, L_r}(\epsilon_r)}\left(\frac{\epsilon}{4L}\right) \right]$.
Then we show that
$
    \mathcal{N}_{ \mathcal{H}}\left(\frac{\epsilon}{4L}\right) {\ge} {d - r \choose r} \mathcal{N}_{  \mathcal{H}_{\mathcal{D}_X, L_r}(\epsilon_r) }\left(\frac{\epsilon}{4L}\right)
$ (Proof in Lemma~\ref{lem:auto-encoder} of Appendix~\ref{app:example_auto})
since
\begin{align*}
    \mathcal{H}_{\mathcal{D}_X, L_r}(\epsilon_r) & = \{
    h_W(x): W = OB_{1:r}^\top, O \in \mathbb{R}^{r \times r}, O \textrm{~is orthonormal}\},
    \\
    \mathcal{H} & \supseteq \{h_W(x): W=O B_S^\top, O \in \mathbb{R}^{r \times r}, O \textrm{~is orthonormal}, S \subseteq \{r+1, \ldots, d\}, |S|{=}r\},
\end{align*}
where $B_S$ refers to the sub-matrix of columns in $B$ having indices in $S$. Therefore, the label sample complexity bound is reduced by $\frac{C}{\epsilon^2}\ln {d - r\choose r}$, i.e.,
the error bound is reduced by  $\frac{C}{\sqrt{m_\ell}}\ln {d -r \choose r}$ when using $m_\ell$ labeled points.
Note that $\ln {d - r \choose r} = \Theta(r \ln (d/k))$ when $r$ is small, and thus the reduction is roughly linear initially and then grows slower with $r$. 
Interestingly, the reduction depends on the hidden dimension $r$ but has little dependence on the input dimension $d$. 

\subsection{An Example of Functional Regularization via Masked Self-supervision}
\label{subsec:masked_example}
 
\myparagraph{Learning Without Functional Regularization.}
Let $\mathcal{H}$ be linear functions from $\mathbb{R}^d$ to $\mathbb{R}^r$ where $r < (d-1)/2 $ followed by a quadratic activation, and $\mathcal{F}$ be linear functions from $\mathbb{R}^r$ to $\mathbb{R}$. That is,
\begin{align}
\label{eqn:masked}
\phi = h_W(x) = [\sigma(w_1^\top x), \ldots, \sigma(w_r^\top x)] \in \mathbb{R}^r \ ,\ y = f_a(\phi) = a^\top \phi, \textrm{~where~} w_i\in \mathbb{R}^{d} \ ,  a\in \mathbb{R}^{r}.
\end{align}
Here $\sigma(t) {=} t^2$ for $t \in \mathbb{R}$ is the quadratic activation function.
W.l.o.g, we assume that $w_i$ and $a$ have $\ell_2$ norm bounded by $1$. 
Without prior knowledge on the data, no functional regularization corresponds to end-to-end training on $\mathcal{F} {\circ} \mathcal{H}$. 

\myparagraph{Data Property.}
We consider the setting where the data point $x$ satisfies $x_1 = \sum_{i=1}^r ((u_i^*)^\top x_{2:d})^2$, where $x_{2:d} = [x_2, x_3, \ldots, x_d]$ and $u_i^*$ is the $i$-th eigenvector of $\Sigma := \mathbb{E}[x_{2:d} x_{2:d}^\top]$. Furthermore, the label $y$ is given by $y = \sum_{i=1}^r a^*_i \sigma((u_i^*)^\top x_{2:d}) + \nu$ for some $\|a^*\|_2 {\le} 1$ and a small Gaussian noise $\nu$.
We also assume a difference in the ${r}^{\text{th}}$ and ${r{+}1}^{\text{th}}$ eigenvalues of $\Sigma$.

\myparagraph{Learning With Functional Regularization.}
Suppose we have prior knowledge that $x_1 = \sum_{i=1}^r (u_i^\top x_{2:d})^2$ and $y = \sum_{i=1}^r a_i \sigma(u_i^\top x_{2:d})$ for some vectors $u_i \in \mathbb{R}^{d-1}$ and an $a$ with $\|a\|_2 \le 1$. Based on this, we perform masked self-supervision by constraining the first coordinate of $w_i$ to be $0$ for $h$, and choosing the regularization function $g(\phi) {=} \sum_{i=1}^r \phi_i$ and the regularization loss $L_r(h, g; x) {=} (x_1 {-} g(h_W(x)))^2$. Again for simplicity, we assume access to infinite unlabeled data and set the regularization loss threshold $\tau=0$.

On applying our framework, we get that functional regularization can reduce the labeled sample bound by $\frac{C}{\epsilon^2} \left[ \ln \mathcal{N}_{ \mathcal{H}}\left(\frac{\epsilon}{4L}\right) {-} \ln \mathcal{N}_{  \mathcal{H}_{\mathcal{D}_X, L_r}(0)}\left(\frac{\epsilon}{4L} \right)\right]$ for some absolute constant $C$.  
We can derive the following using properties of $L_r$ and $g$ (Proof in Lemma~\ref{lem:masked} of Appendix~\ref{app:example_masked}):
\begin{align*}
    \mathcal{H}_{\mathcal{D}_X, L_r}(0) {=} \{
    h_W(x): w_i {=} [0,u_i], [u_1{, \ldots,} u_r]^\top {=} O [u_1^*, \ldots, u_r^*]^\top, O \in \mathbb{R}^{r \times r}, O \textrm{~is orthonormal}\}
\end{align*}
Using this we can show that the reduction of the sample bound is $\frac{C}{\epsilon^2}\ln {d-1 - r \choose r}$, i.e., a reduction in the error bound by  $\frac{C}{\sqrt{m_\ell}} \ln {d-1 - r \choose r} $ when using $m_\ell$ labeled data. We also note that this reduction depends on $r$ but has little dependence on $d$.

\section{Experiments} \label{sec:exp}

There is abundant empirical evidence on the benefits of auxiliary tasks in various applications, and hence we present experiments on verifying the benefits of functional regularization in Appendix~\ref{app:additional_expts}. Here we focus on experimentally verifying the following implications for the two concrete examples that we have analysed under our framework: 1) the reduction in prediction error (between end-to-end training and functional regularization) using the same amount of labeled data; 2) the reduction in prediction error on varying a property of the data and hypotheses (specifically, varying parameter $r$); 3) the pruning of the hypothesis class which results in reducing the prediction error. We present the complete experimental details in Appendix~\ref{app:experiments} for reproducibility, and additional results which verify that the reduction has little dependence on the input dimension $d$.  

\myparagraph{Setup.}
For the auto-encoder example, we randomly generate $d$ orthonormal vectors ($\{u^*_i\}_{i=1}^{i=d}$) in $\mathbb{R}^{d}$, means $\mu_i$ and variances $\sigma_i$ for $i\in[d]$ such that $\sigma_1 {> \dots >} \sigma_r \gg \sigma_{r+1} {> \dots >} \sigma_d$. We sample $\alpha_i {\sim} \mathcal{N}(\mu_i,\sigma_i) \ \forall i\in[d]$ and generate $x=\sum_{i=1}^{d} \alpha_i u_i$. For generating $y$, we use a randomly generated vector $a^* \in \mathbb{R}^{r}$. We use $d=100$ and generate $10^4$ unlabeled, $10^4$ labeled training and $10^3$ labeled test points. We use the quadratic activation function and follow the specification in Section~\ref{subsec:vae_example} (with regards to the hypothesis classes, reconstruction losses, etc.). For the masked self-supervision example, we similarly generate $x$ having the data property specified in Section~\ref{subsec:masked_example} and then follow the other specifications in Section~\ref{subsec:masked_example} (with regards to the hypothesis classes, reconstruction losses, etc.). We report the MSE on the test data points averaged over 10 runs as the metric.

To support our key theoretical intuition that functional regularization prunes the hypothesis classes, we seek to visualize the learned model. Since multiple permutations of model parameters can result in the same model behavior, we visualise the function (from input to output) represented instead of the parameters using the method from~\cite{10.5555/1756006.1756025}. Formally, we concatenate the outputs $f(h(x))$ on the test set points from a trained model into a single vector and visualise the vector in 2D using t-SNE~\cite{vanDerMaaten2008}. We perform 1000 independent runs for each of the two models (with and without functional regularization) and plot the vectors for comparison. See Appendix~\ref{app:auto_encoder_experiments} for more details.

\begin{figure}[t]
\centering
\subfigure[]{\label{fig:1}\includegraphics[width=.32\linewidth]{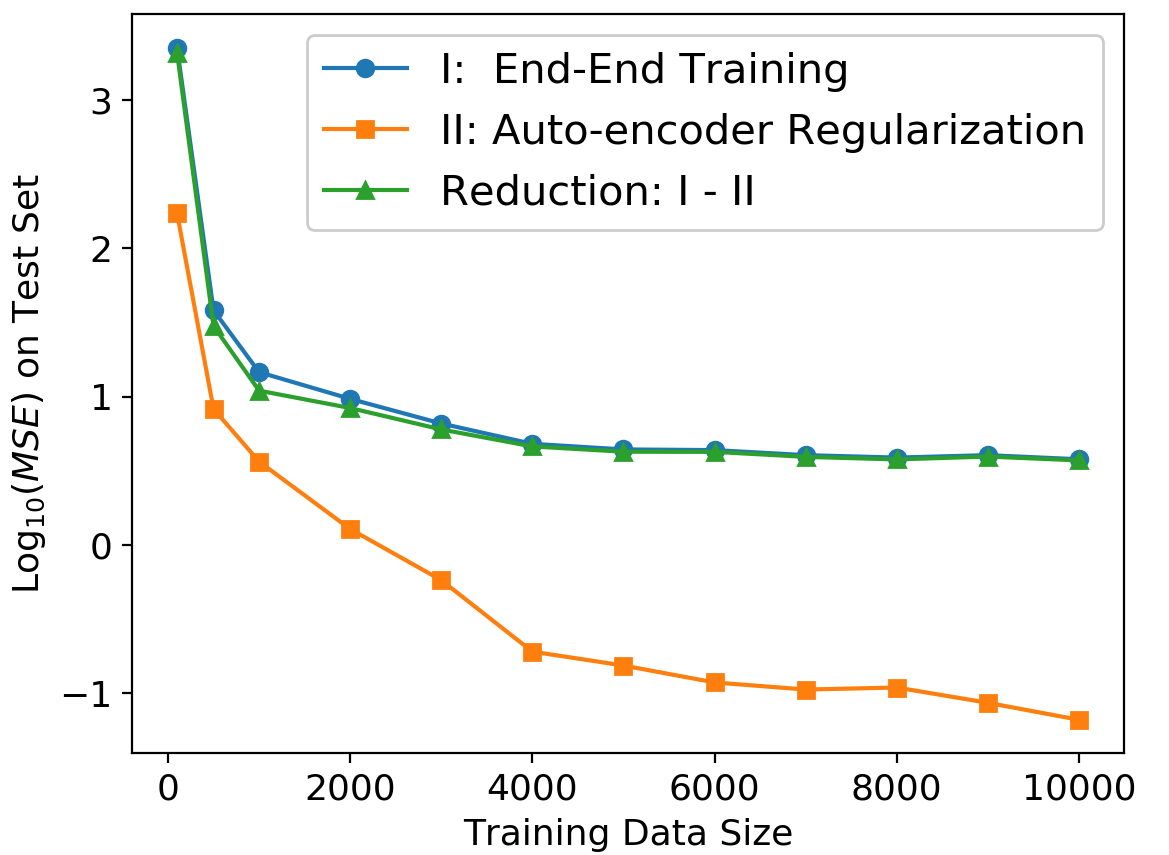}}
\subfigure[]{\label{fig:2}\includegraphics[width=.32\linewidth]{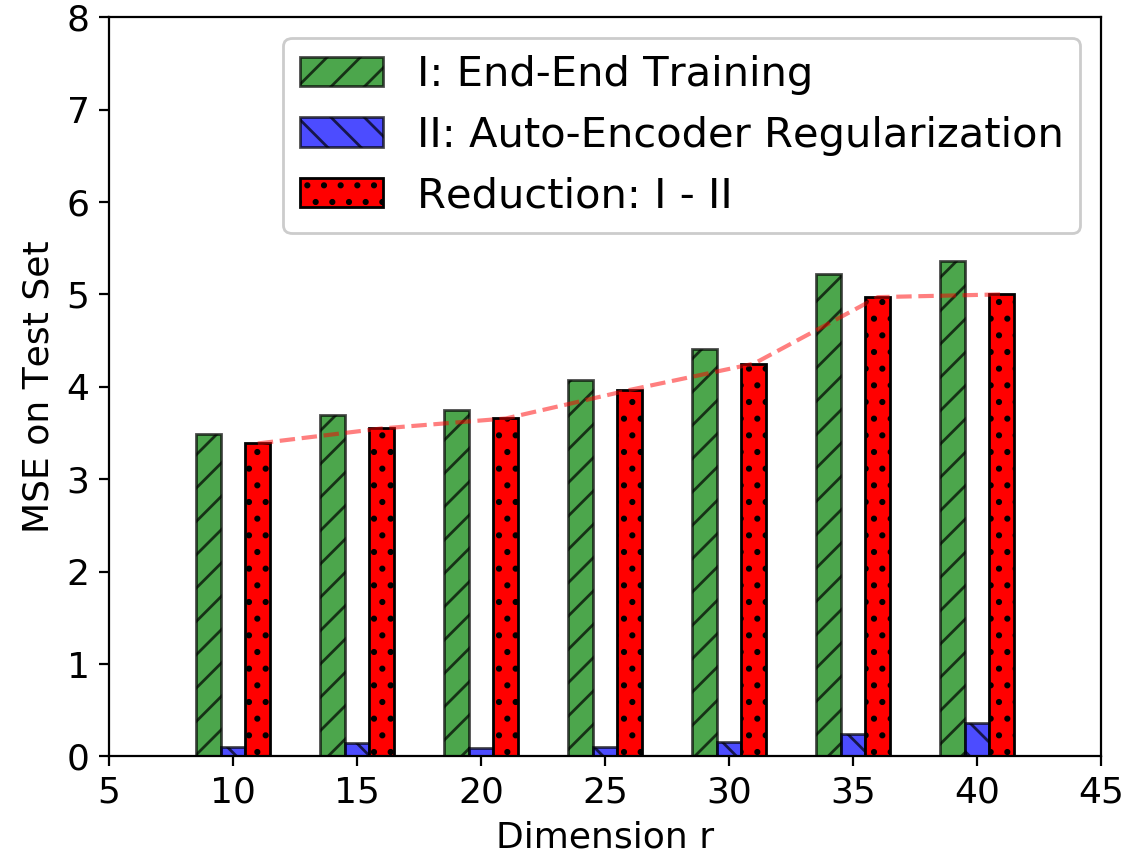}}
\subfigure[]{\label{fig:3}\includegraphics[height=0.23\linewidth, width=.32\linewidth]{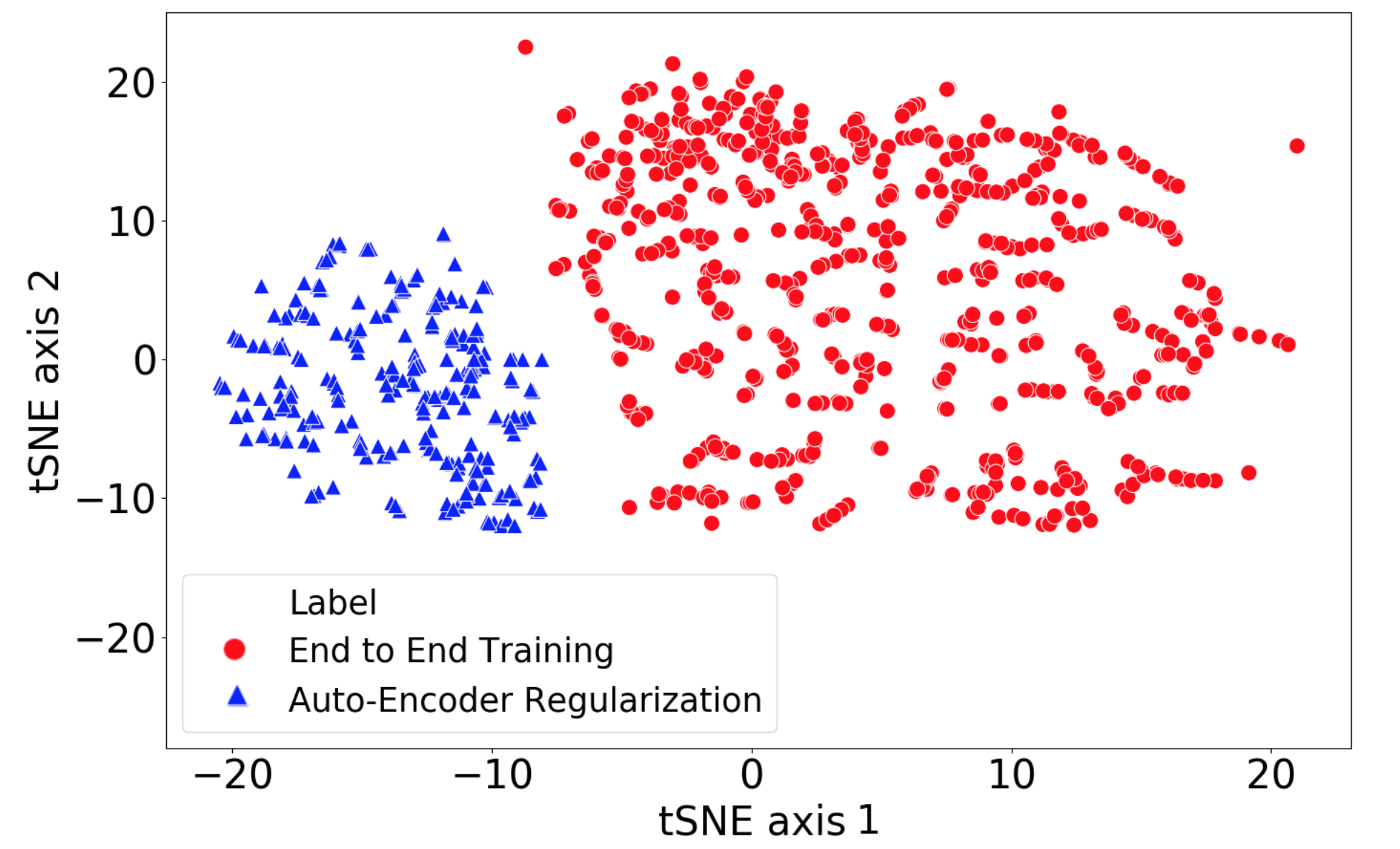}} \\
\vspace{-1em}
\subfigure[]{\label{fig:4}\includegraphics[width=.32\linewidth]{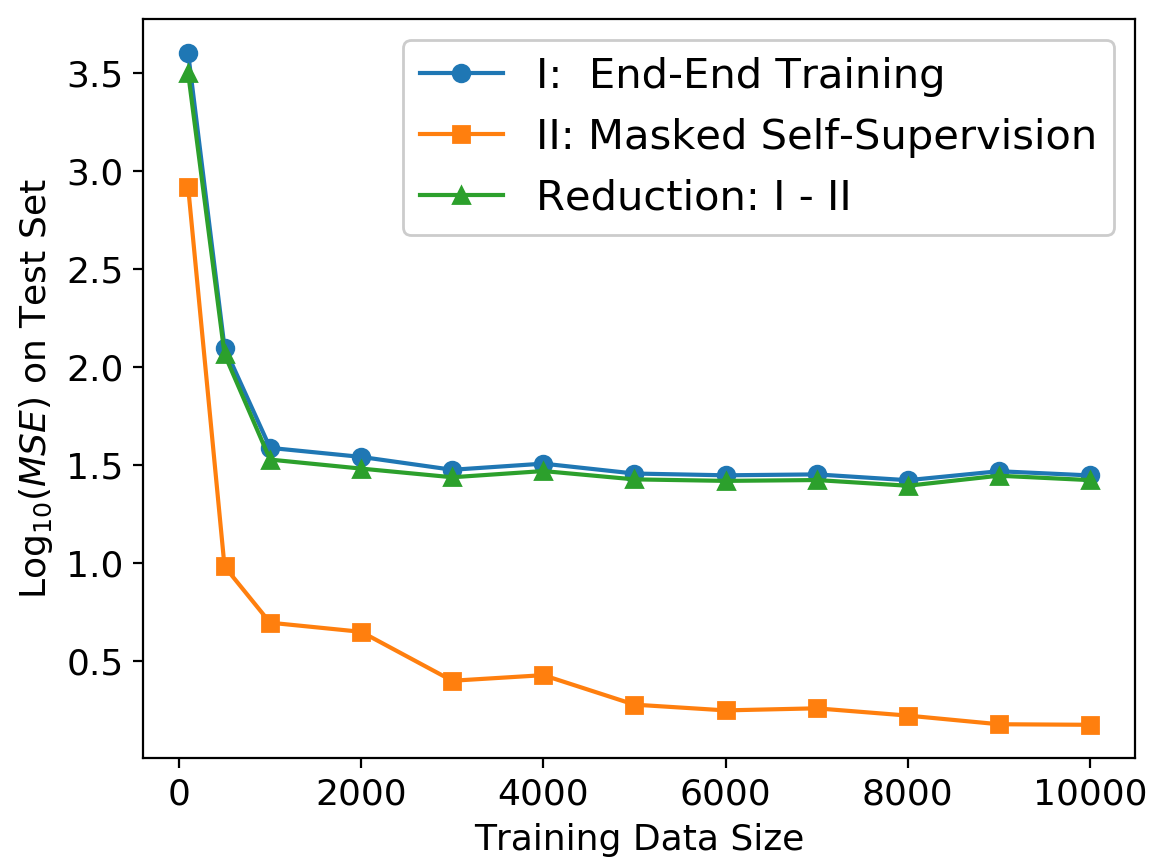}}
\subfigure[]{\label{fig:5}\includegraphics[width=.32\linewidth]{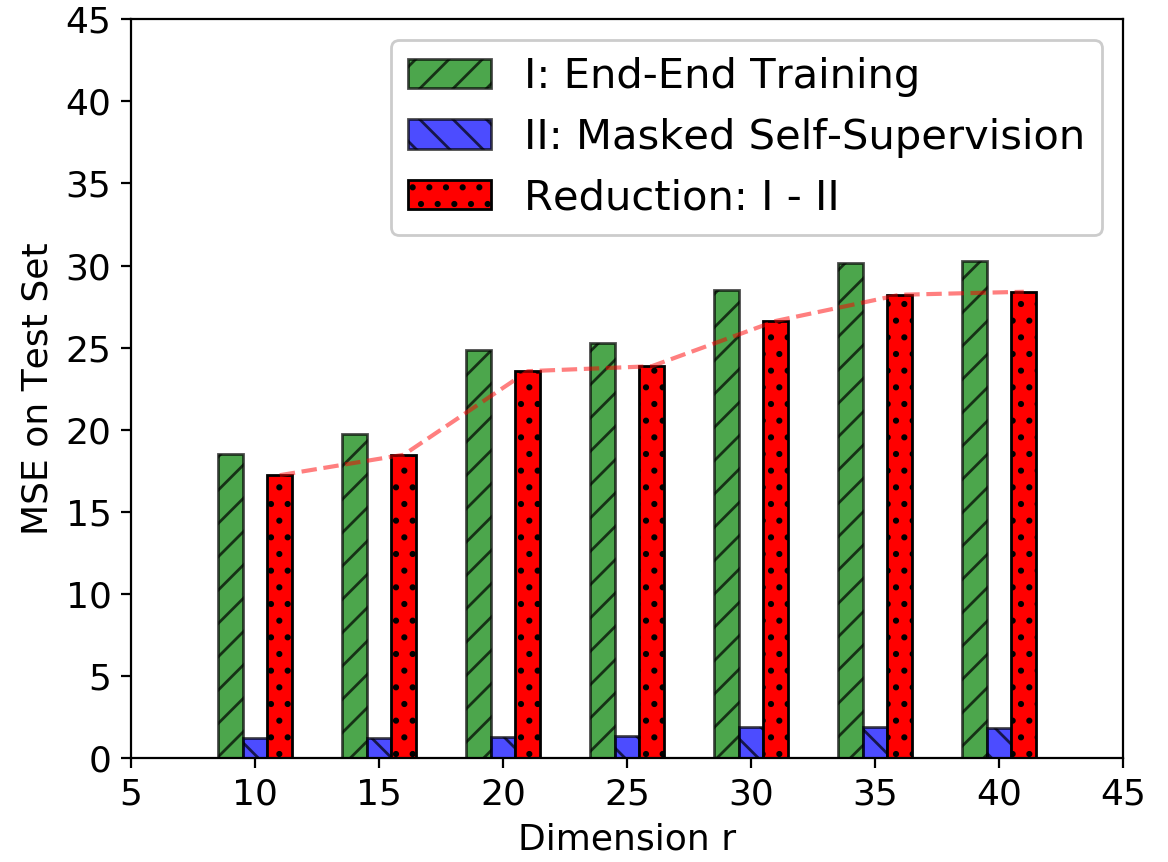}}
\subfigure[]{\label{fig:6}\includegraphics[height=0.23\linewidth, width=.32\linewidth]{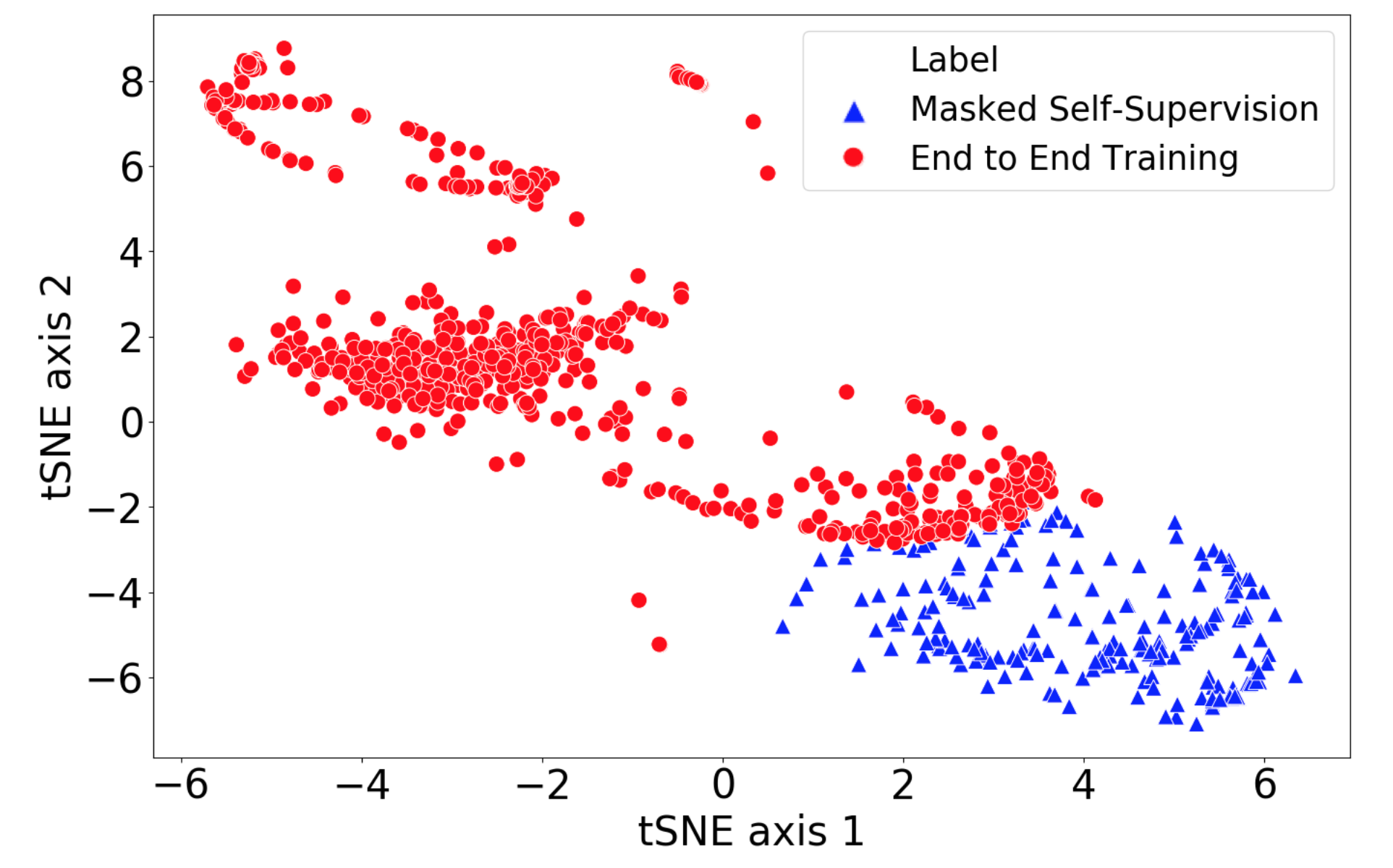}} \\
\vspace{-1em}
\caption{\textbf{Auto-Encoder:} \ref{fig:1} shows the variation of test MSE with labeled data size (here $r$=30), \ref{fig:2} shows this variation with the parameter $r$, and \ref{fig:3} shows the 2D visualization of the functional approximation using t-SNE. \textbf{Masked Self-Supervision:} \ref{fig:4}, \ref{fig:5} and \ref{fig:6} show the same corresponding plots. Reduction refers to Test MSE of end-to-end training - Test MSE with regularization.}
\end{figure} 

\myparagraph{Results.}
Figure~\ref{fig:1} plots the Test MSE loss by varying the size of the labeled data when $r=30$. We observe that with the same labeled data size, functional regularization can significantly reduce the error compared to end-to-end training. Equivalently, it needs much fewer labeled samples to achieve the same error as end-to-end training (e.g., 500 v.s.\ 10,000 points). 
Also, the error without regularization does not decrease for sample sizes $\ge 2000$ while it decreases with regularization, suggesting that the regularization can even help alleviate optimization difficulty.
Figure~\ref{fig:2} shows the effect of varying $r$ (i.e., the dimension of the subspace containing signals for prediction). We observe that the reduction in the error increases roughly linearly with $r$ and then grows slower, as predicted by our analysis.
Figure~\ref{fig:3} visualizes the prediction functions learned. It shows that when using the functional regularization, the learned functions stay in a small functional space, while they are scattered otherwise. This supports the intuition behind our theoretical analysis. This result also interestingly suggests that pruning the representation hypothesis space via functional regularization translates to a compact functional space for the prediction, even through optimization.
We make similar observations for the masked self-supervision example in Figure~\ref{fig:4}-\ref{fig:6}, which provide additional support for our analysis.

\section{Conclusion}
In this paper we have presented a unified discriminative framework for analyzing several representation learning approaches using unlabeled data, by viewing them as imposing a regularization on the representation via a learnable function. We have derived sample complexity bounds under various assumptions on the hypothesis classes and data, and shown that the functional regularization can be used to prune the hypothesis class and reduce the labeled sample complexity. We have also applied our framework to two concrete examples. An interesting future work direction is to investigate the effect of such functional regularization on the optimization of the learning methods. 

\section*{Broader Impact}
Our paper is mostly theoretical in nature and thus we foresee no immediate negative societal impact. We are of the opinion that our theoretical framework may inspire development of improved representation learning methods on unlabeled data, which may have a positive impact in practice. In addition to the theoretical machine learning community, we perceive that our precise and easy-to-read formulation of unsupervised learning for downstream tasks can be highly beneficial to engineering-inclined machine learning researchers.

\section*{Funding Disclosure}
The work in this paper was funded in part by FA9550-18-1-0166 and IIS-2008559. The authors would also like to acknowledge the support provided by the University of Wisconsin-Madison Office of the Vice Chancellor for Research and Graduate Education with funding from the Wisconsin Alumni Research Foundation. No other funding entities with any competing interests influenced our work.

\section*{Acknowledgements}
The authors would like to thank the anonymous reviewers and the meta-reviewer for their valuable comments and suggestions which have been incorporated for the camera ready version.

\bibliographystyle{abbrv}
\bibliography{references}

\clearpage
\appendix
\begin{center}
    \textbf{\LARGE Appendix}
\end{center}

\section{Instantiations of Functional Regularization} \label{app:instance}
Here we show that several unsupervised (self-supervised) representation learning strategies can be viewed as imposing a learnable function to regularize the representations being learned.
We note that the class $\mathcal{G}$ can be an index set instead of a class of functions; our framework applies as long as the loss $L_r(h,g;x)$ is well defined (see the manifold learning example). $\mathcal{G}$ can also only have a single $g$, corresponding to the special case of a fixed regularizer (see the $\ell_p$ norm penalty example). 

\paragraph{Auto-encoder.}
Auto-encoders use an encoder function $h$ to map the input $x$ to a lower dimensional space $\phi$ and a decoder network $d$ to reconstruct the input back from $\phi$ using a MSE loss $\|x-d(h(x))\|^2$. One can view $d$ as a regularizer on the feature representation $\phi=h(x)$ through the regularization loss $L_r(h,g\ ; x)=\|x-d(h(x))\|^2$. $\mathcal{H}_{\mathcal{D}_X, L_r}(\tau)$ is the subset of representation functions with at most $\tau$ reconstruction error using the best decoder in $\mathcal{G}$.

Variants of standard auto-encoders like sparse auto-encoders can be formulated similarly as a functional regularization on the representation being learnt. 

\paragraph{Masked Self-supervised Learning.}
Masked self-supervision techniques, in abstract terms, cover a portion of the input and then predict the masked input portion~\cite{devlin-etal-2019-bert}. More concretely, say the input $x=[x_1, x_2 ,\dots,x_d]$ is masked as $x'=[x_1,\dots,x_i,0,\dots,0,x_j,\dots,x_d]$ and a function $g$ is learned to predict the masked input $[x_{i+1},\dots,x_{j-1}]$ over an input representation $h(x)$. This function $g$ used to reconstruct $x$, can be viewed as imposing a regularization on $h$ through a MSE regularization loss given by $\|x_{[i+1:j-1]} - g(h(x'))\|^2$. $\mathcal{H}_{\mathcal{D}_X, L_r}(\tau)$ is the subset of $\mathcal{H}$ which have at most $\tau$ MSE on predicting $x_{[i+1:j-1]}$ using the best function $g \in G$.

\paragraph{Variational Auto-encoder.} VAEs encode the input $x$ as a distribution $q_{\phi}(z|x)$ over a parametric latent space $z$ instead of a single point, and sample from it to reconstruct $x$ using a decoder $p_{\theta}(x|z)$. The encoder $q_{\phi}(z|x)$ is used to model the underlying mean $\mu_z$ and co-variance matrix $\sigma_z$ of the distribution over $z$. VAEs are trained by minimising a loss
$$
\mathcal{L}_{x}(\theta,\phi) = -\mathbb{E}_{z \sim q_{\phi}(z|x)}[\log \ p_{\theta}(x|z)] + \mathbb{KL}(q_{\phi}(z|x)||p(z))
$$ 
where $p(z)$ is specified as the prior distribution over $z$ (e.g., $\mathcal{N}(0,1)$). The encoder $q_{\phi}(z|x)$ can be viewed as the representation function $h$, the decoder $p_{\theta}(x|z)$ as the learnable regularization function $g$, and the loss $\mathcal{L}_{x}(\theta,\phi)$ as the regularization loss $L_r(h, g; x)$ in our framework. Then $\mathcal{H}_{\mathcal{D}_X, L_r}(\tau)$ is the subset of encoders $q_{\phi}$ which have at most $\tau$ VAE loss when using the best decoder $p_\theta$ for it.

\paragraph{Manifold Learning through the Triplet Loss.} Learning manifold representations through metric learning is a popular technique used in computer vision applications~\cite{journals/corr/HofferA14}. A triplet loss formulation is used to learn a distance metric for the representations, by trying to minimise this metric between a baseline and positive sample and maximising the metric between the baseline and a negative sample. This is achieved by learning a representation function $h$ for an input $x$. Considering a triple of input samples $\bar{x} = (x_b,x_p,x_n)$ corresponding to a baseline, positive and negative sample, we use a loss $L_\mathrm{Triplet}(\bar{x})=\max(\|h(x_b)-h(x_p)\|_2^2 - \|h(x_b)-h(x_n)\|_2^2,0)$ to learn $h$. This is a special instantiation of our framework using a dummy $\mathcal{G}$ having a single function $g$, where the regularization loss $L_r(h,g\ ; \bar{x})=L_\mathrm{Triplet}(\bar{x})$ is computed over a triple of input samples.

Further, one can also consider some variants of the standard triplet loss formulation under our functional regularization perspective. 
For example, let the triplet loss be $L^{(\alpha)}_\mathrm{Triplet}(\bar{x}) {=} \max(\|h(x_b){-}h(x_p)\|_2^2 {-} \|h(x_b){-}h(x_n)\|_2^2 {+} \alpha,0)$ where $\alpha \in \mathbb{R}$ is a margin between the positive and negative pairs. When $\alpha$ is learnable, this corresponds to a functional regularization where $\mathcal{G} = \{\alpha: \alpha \in \mathbb{R}\}$, and the regularization loss is $L_r(h,g; \bar{x}) = L^{(\alpha)}_\mathrm{Triplet}(\bar{x})$. In this case, the class $\mathcal{G}$ is not defined on top of the representation $h(x)$. However, our framework and the sample complexity analysis can still be applied through the definition of $L_r(h,g; \bar{x})$. 

\paragraph{Sparse Dictionary Learning.}
Sparse dictionary learning is an unsupervised learning approach to obtain a sparse low-dimensional representation of the input data.
Here we consider a distributional view of sparse dictionary learning.
Give a distribution $\mathcal{D}_X$ over unlabeled data $x \in \mathbb{R}^d$ and a hyper-parameter $\lambda > 0$, we want to find a dictionary matrix $D \in \mathbb{R}^{d \times K}$ and a sparse representation $z \in \mathbb{R}^K$ for each $x$, so as to minimize the error $\mathbb{E}[L_{D}(x)]$, where $L_{D}(x)$ is the error on one point $x$ defined as $L_{D}(x) := \|x -D z \|_{2}^{2}+\lambda \|z\|_{0}$, subject to the constraint that each column of $D$ has $\ell_2$ norm bounded by $1$. The learned representations $z$ can then be used for a target prediction task.
Under our framework, we can view the representation function corresponding to $z = h_D(x) = \arg\min_{z \in \mathbb{R}^K} \|x -D z \|_{2}^{2}+\lambda \|z\|_{0}$, and $D$ is the parameter of the representation function. The regularization function class $\mathcal{G} $ has a single $g$, and the regularization loss is $L_r(h_D, g ; x) = L_{D}(x)$.

Our framework also captures an interesting variant of dictionary learning. Consider another dictionary matrix $F$ and a hyper-parameter $\eta > 0$. The representation function still corresponds to $z = h_D(x) = \arg\min_{z \in \mathbb{R}^K} \|x -D z \|_{2}^{2}+\lambda \|z\|_{0}$, with $D$ as the parameter. The regularization function class is now given by $\mathcal{G} = \{g_F(z) = Fz : F \in \mathbb{R}^{d \times K}\}$, and the regularization loss $L_r(h_D, g_F; x)$ is defined as $ \|x - g_F( h_D(x) ) \|_{2}^{2}+\lambda \|z\|_{0} + \eta \| D - F\|^2_F$. This special case of dictionary learning allows the encoding and decoding steps to use two different dictionaries $D$ and $E$ but constraining the difference between them. When $\eta \rightarrow +\infty$, this variant reduces to the original version described earlier.

\paragraph{Explicit $\ell_p$ Norm Penalty.}
Techniques imposing explicit regularizations on the representation $h$ being learned, often use an $\ell_p$ norm penalty on $h(x)$ i.e, $\|h(x)\|_p^{p}$ to the prediction loss while jointly training $f$ and $h$. This can be viewed as a special case of our framework using a fixed regularization function $g(h(x))=\|h(x)\|_p^{p}$.

\paragraph{Restricted Boltzmann Machines.} Restricted Boltzmann Machines (RBM)~\cite{smolensky1986information,hinton2006reducing} generate hidden representations for an input through unsupervised learning on unlabeled data. 
RBMs are characterized by a joint distribution over the input $x \in \{0, 1\}^d$ and the representation $z \in \{0,1\}^r$: 
$P(x, z) = \frac{1}{Z} e^{-E(x,z)}$,
where $Z$ is the partition function and $E(x,z)$ is the energy function defined as: $E(x,z) {=} {-} a^\top x {-} b^\top z {-} x^\top W z$, where $a \in \mathbb{R}^{d}, b \in \mathbb{R}^r, W \in \mathbb{R}^{d \times r}$ are parameters to be learned. 

Then $P(z|x)$, for a fixed $x$, is a distribution parameterized by $b$ and $W$; which can be denoted as $q_{W,b}(z|x)$. Similarly, $P(x|z)$ is parameterized by $a$ and $W$ and thus can be denoted as $p_{W,a}(x|z)$. 
Given $x \sim \mathcal{D}_X$, the objective of the RBM is to minimize $- \mathbb{E}_{x \sim \mathcal{D}_X} [\log P(x)]$. 

While the standard RBM objective does not have a direct analogy under our functional regularization framework, a heuristic variant can be formulated under our framework.
If we use $\mathbb{E}_{P(x)}$ to denote the expectation over the marginal distribution of $x$ in the RBM, $\mathbb{E}_{P(z)}$ to denote the expectation over the marginal distribution of $z$, and $\mathbb{E}_{\mathcal{D}_X}$ to denote the expectation over $x \sim \mathcal{D}_X$. Then the following hold for the standard RBM: 
\begin{align}
P(z) & =  \mathbb{E}_{P(x)} [P(z|x)] = \mathbb{E}_{P(x)} [q_{W,b}(z|x)] \label{eqn:pz}
\\
P(x) & =  \mathbb{E}_{P(z)} [P(x|z)] = \mathbb{E}_{P(z)} [p_{W,a}(x|z)]
\end{align}
In the heuristic variant, we replace $P(x)$ with $\mathcal{D}_X$ in Equation~\eqref{eqn:pz}:
\begin{align}
\begin{split}
\hat{P}(z) =  \mathbb{E}_{\mathcal{D}_X} [P(z|x)] = \mathbb{E}_{\mathcal{D}_X} [q_{W,b}(z|x)],
\end{split}
\begin{split}
\hat{P}(x) =  \mathbb{E}_{\hat{P}(z)} [P(x|z)] = \mathbb{E}_{\hat{P}(z)} [p_{W,a}(x|z)],
\end{split}
\end{align}
and train using the loss:
\begin{align}
    L(W, a, b; x) := {-}\log \hat{P}(x) {=} {-}\log \mathbb{E}_{\hat{P}(z)} \LARGE\{ p_{W,a}(x|z) \LARGE\} {=} {-}\log \mathbb{E}_{\mathbb{E}_{\mathcal{D}_X} [q_{W,b}(z|x)]} \LARGE\{ p_{W,a}(x|z) \LARGE\}.
\end{align}
Furthermore, on introducing another weight matrix $F \in \mathbb{R}^{d \times r}$ for $P(x|z)$ and a hyper-parameter $\eta>0$, we can train the RBM using the loss:
\begin{align}
    L_\eta(W, a, b; x) :=   -\log \mathbb{E}_{\mathbb{E}_{\mathcal{D}_X} [q_{W,b}(z|x)]} \LARGE\{ p_{F,a}(x|z) \LARGE\} + \eta \| W  -  F\|_F^2.
\end{align}
When $\eta \rightarrow +\infty$, this loss function reduces to the loss $L(W, a, b; x)$. Here $q_{W,b}(z|x)$ can be viewed as the representation function $h$ of our framework, $p_{F,a}(x|z)$ as the regularization function $g$, and $L_\eta(W, a, b; x)$ as the regularization loss $L_r(h, g; x)$.

\paragraph{Comparison to GANs.} Finally, we would like to comment on Generative Adversarial Networks (GANs)~\cite{goodfellow2014generative}. While both functional regularization and GANs use auxiliary tasks having a function class, the goal of GANs is to learn a generative model using an auxiliary task through a discriminative function (the discriminator), while the goal of functional regularization is to learn a discriminative model using an auxiliary task which is usually (though not always) through a generative function (e.g., the decoder in auto-encoders).

\section{Sample Complexity Bounds} \label{app:bound}

\subsection{Same Domain, Realizable, Finite Hypothesis Classes}
For simplicity, we begin with the realizable case, where the hypothesis classes contain functions $g^*, h^*, f^*$ with a zero prediction and regularization loss. Here we consider that the unlabeled $U$ and labeled $S$ samples are from the same domain distribution $\mathcal{D}_X$. We derive the following Theorem.

\samerealizable*
\begin{proof}
We first show that with high probability, only the hypotheses $h$ in $\mathcal{H}_{\mathcal{D}_X, L_r} (\epsilon_0)$ have $L_r(h; U) = 0$.
For a given pair $g$ and $h$ with $L_r(h, g; \mathcal{D}_X) \ge \epsilon_0$, the probability that $L_r(h, g; U) = 0$ is at most
\begin{align}
    \mathbb{P}[L_r(h, g; U) = 0] \le (1-\epsilon_0)^{m_u} \le \frac{\delta}{2 |\mathcal{H}||\mathcal{G}|}
\end{align}
for the given value of $m_u$.
By the union bound, with probability at least $1 - \delta/2$, only those $g$ and $h$ with $L_r(h, g; \mathcal{D}_X) \le \epsilon_0$ have $L_r(h, g; U) = 0$. Then only hypotheses $h \in \mathcal{H}_{\mathcal{D}_X, L_r}(\epsilon_0)$ have $L_r(h; U) = 0$.

Then we show that with high probability, for all $h \in \mathcal{H}_{\mathcal{D}_X, L_r} (\epsilon_0)$, only those $f$ and $h$ with $L_c(f, h;\mathcal{D}) \le \epsilon_1$ can have $L_c(f,h; S) = 0$.
Similarly as above, for a pair $f \in \mathcal{F}$ and $h \in \mathcal{H}_{\mathcal{D}_X, L_r} (\epsilon_0)$ with $L_c(f, h;\mathcal{D}) \ge \epsilon_1$, the probability that $L_c(f, h; S) = 0$ is at most
\begin{align}
    \mathbb{P}[L_c(f, h; S) = 0] \le (1-\epsilon_1)^{m_\ell} \le \frac{\delta}{2 |\mathcal{H}_{\mathcal{D}_X, L_r} (\epsilon_0)||\mathcal{G}|}
\end{align}
for the given value of $m_\ell$.
By the union bound, with probability $1 - \delta/2$, for $f \in \mathcal{F}$ and $h \in \mathcal{H}_{\mathcal{D}_X, L_r} (\epsilon_0)$, only those with $L_c(f, h;\mathcal{D}) \le \epsilon_1$ can have $L_c(f, h; S) = 0$, proving the theorem.
\end{proof}

\subsection{Same Domain, Unrealizable Case, Infinite Hypothesis Classes}
\label{app:same_domain_unrealizable}

When the hypothesis classes are of an infinite size, we use metric entropy to measure the capacity.
Suppose $\mathcal{H}$ is indexed by parameter set $\Theta_H$  with norm $\|\cdot\|_H$,  $\mathcal{G}$ by $\Theta_G$ with norm $\|\cdot\|_G$, and $\mathcal{F}$ by $\Theta_F$  with norm $\|\cdot\|_F$. 
Assume that the losses are $L$-Lipschitz with respect to the parameters. That is,
\begin{align*}
    |L_r(h_{\theta}, g; x) - L_r(h_{\theta'}, g; x)| \le L\| \theta - \theta'\|_H,\forall g\in \mathcal{G}, x \in \mathcal{X},
    \\
    |L_r(h, g_{\theta}; x) - L_r(h, g_{\theta'}; x)| \le L\| \theta - \theta'\|_G, \forall h\in \mathcal{H}, x \in \mathcal{X},
    \\
    |L_c(h_{\theta}, f; x) - L_c(h_{\theta'}, f; x)| \le L\| \theta - \theta'\|_H, \forall  f \in \mathcal{F}, x \in \mathcal{X},
    \\
    |L_c(h, f_{\theta}; x) - L_c(h, f_{\theta'}; x)| \le L\| \theta - \theta'\|_G,  \forall  h\in \mathcal{H}, x \in \mathcal{X}.
\end{align*}
Let $\mathcal{N}_\mathcal{G}(\epsilon)$ be the $\epsilon$-covering number of $\mathcal{G}$ w.r.t.\ the associated norm. This is similarly defined for the other function classes.

The assumptions that the regularization and prediction losses are 0 are usually impractical due to noise in the data distribution. Realistically we may assume that there exist ground-truth functions that can make the regularization and prediction losses small. We begin by considering a setting where the prediction loss can be zero while the regularization loss is not.

\begin{thm}
Suppose there exist $h^* \in \mathcal{H}, f^* \in \mathcal{F}, g^* \in \mathcal{G}$ such that  $L_c(f^*, h^*; \mathcal{D}) = 0$ and $L_r(h^*, g^*; \mathcal{D}_X) \le \epsilon_r$.
For any $\epsilon_0, \epsilon_1 \in (0, 1/2)$, a set $U$ of $m_u$ unlabeled examples and a set $S$ of $m_l$ labeled examples is sufficient to learn to an error $\epsilon_1$ with probability $1 - \delta$, where
\begin{align}
    m_u & \ge \frac{C}{\epsilon^2_0} \ln\frac{1}{\delta} \left[ \ln\mathcal{N}_{\mathcal{G}}\left(\frac{\epsilon_0}{4L}\right) + \ln\mathcal{N}_{\mathcal{H} }\left(\frac{\epsilon_0}{4L}\right) \right],
    \\
    m_l & \ge \frac{C}{\epsilon_1} \ln\frac{1}{\delta} \left[ \ln\mathcal{N}_{\mathcal{F} }\left(\frac{\epsilon_1}{4L}\right) + \ln\mathcal{N}_{ \mathcal{H}_{\mathcal{D}_X, L_r}(\epsilon_r {+} \epsilon_0) }\left(\frac{\epsilon_1}{4L}\right) \right]  
\end{align}
for some absolute constant $C$.
In particular, with probability at least $1 - \delta$, the hypotheses $f \in \mathcal{F}, h \in \mathcal{H}$ with $L_c(f, h; S) = 0$ and $L_r(h, g; U) \le \epsilon_r + \epsilon_0$ for some $g \in \mathcal{G}$ satisfy
$
L_c(f, h; \mathcal{D}) \le \epsilon_1.
$
\end{thm}
\begin{proof}
First, we show that with $m_u$ unlabeled examples, by a covering argument over $\mathcal{H}$ and $\mathcal{G}$ (see, e.g.,~\cite{vershynin2018high}), it is guaranteed that with probability $1 - \delta/2$, all $h \in \mathcal{H}$ and $g \in \mathcal{G}$ satisfy $|L_r(h, g; U) - L_r(h, g; \mathcal{D}_X) | \le \epsilon_0$. 
More precisely, let $\mathcal{C}_{\mathcal{G}}\left(\frac{\epsilon_0}{4L}\right)$ be a $\frac{\epsilon_0}{4L}$-covering of $\mathcal{G}$, and $\mathcal{C}_{\mathcal{H} }\left(\frac{\epsilon_0}{4L}\right)$ be a $\frac{\epsilon_0}{4L}$-covering of $\mathcal{H}$. Then by the union bound, all $h' \in \mathcal{C}_{\mathcal{H} }\left(\frac{\epsilon_0}{4L}\right)$ and $g' \in \mathcal{C}_{\mathcal{G} }\left(\frac{\epsilon_0}{4L}\right)$ satisfy $|L_r(h', g'; U) - L_r(h', g'; \mathcal{D}_X) | \le \epsilon_0/4$. Then the claim follows from the definition of the coverings and the Lipschitzness of the losses.

By the claim, we have $L_r(h^*, g^*; U) \le L_r(h^*, g^*; \mathcal{D}_X)  +  \epsilon_0 \le \epsilon_r + \epsilon_0$. So $h^* \in \mathcal{H}_{\mathcal{D}_X, L_r}(\epsilon_r + \epsilon_0)$, and thus the optimal value $L_c(f, h; S)$ subject to $L_r(h, g; U) \le \epsilon_r + \epsilon_0$ for some $g \in \mathcal{G}$ is 0.
On the other hand, again by a covering argument over $\mathcal{H}$ and $\mathcal{F}$, with probability at least $1 - \delta/2$, for all $h \in \mathcal{H}_{\mathcal{D}_X, L_r}(\epsilon_r + \epsilon_0)$ and all $f \in \mathcal{F}$, only those with $ L_c(f, h; \mathcal{D}) \le \epsilon_1$ can have $L_c(f, h; S) = 0$. The theorem statement then follows.
\end{proof}

The theorem shows that when the optimal regularization loss is not zero but $\epsilon_r > 0$, one needs to do the learning subject to $L_r(h; U) \le \epsilon_r + \epsilon_0$ and the unlabeled sample complexity has a dependence on $\epsilon_0$ by $\frac{1}{\epsilon_0^2}$, instead of $\frac{1}{\epsilon_0}$.

We are now ready to present the result for the setting where both the optimal prediction and regularization losses are non-zero. 

\thmcovering*

\begin{proof}
With $m_u$ unlabeled examples, by a standard covering argument, it is guaranteed that with probability $1 - \delta/4$, all $h \in \mathcal{H}$ and $g \in \mathcal{G}$ satisfy $|L_r(h, g; U) - L_r(h, g; \mathcal{D}_X) | \le \epsilon_0$. In particular, $L_r(h^*, g^*; U) \le L_r(h^*, g^*; \mathcal{D}_X)  +  \epsilon_0 \le \epsilon_r + \epsilon_0$.
Then again by a covering argument, the labeled sample size $m_l$ implies that with probability at least $1 - \delta/2$, all hypotheses $h \in \mathcal{H}_{\mathcal{D}_X, L_r}(\epsilon_r + 2\epsilon_0)$ and all $f \in \mathcal{F}$ have $L_c(f, h; S) \le L_c(f, h; \mathcal{D}) + \epsilon_1/2$. 
Finally, by using Hoeffding's bounds, with probability at least $1- \delta/4$, we have
$$
L_c(f^*, h^*; S) \le L_c(f^*, h^*; \mathcal{D})  + \mathcal{O}\left(\sqrt{\frac{1}{m_l} \ln \frac{1}{\delta}}\right) \le L_c(f^*, h^*; \mathcal{D})  + \epsilon_1/2.
$$
Therefore, with a probability of at least $1 - \delta$, the hypotheses $f \in \mathcal{F}, h \in \mathcal{H}$ that optimizes $L_c(f, h; S)$ subject
to $L_r(h, g; U) \le \epsilon_r + \epsilon_0$ for some $g \in \mathcal{G}$ have the following guarantee. 
First, since $L_r(h, g; U) \le \epsilon_r + \epsilon_0$, we have $L_r(h, g; \mathcal{D}_X) \le \epsilon_r + 2\epsilon_0$, and thus $h \in \mathcal{H}_{\mathcal{D}_X, L_r}(\epsilon_r + 2\epsilon_0)$. Then we have
\begin{align}
 L_c( f, h; \mathcal{D}) & \le L_c(f, h; S) + \epsilon_1/2
\\
& \le L_c(f^*, h^*; S) + \epsilon_1/2
\\
& \le L_c(f^*, h^*; \mathcal{D}) + \mathcal{O}\left(\sqrt{\frac{1}{m_l} \ln \frac{1}{\delta}}\right) +  \epsilon_1/2
\\
& \le L_c(f^*, h^*; \mathcal{D}) + \epsilon_1.
\end{align}
This completes the proof of the theorem.
\end{proof}

The above analysis also holds with some other capacity measure of the hypothesis classes, like the VC-dimension or Rademacher complexity. We give an example for using the VC-dimension (assuming the prediction task is a classification task). The proof follows similarly to Theorem~\ref{thm:covering}, but using the VC-dimension bound instead of the $\epsilon$-net argument.
\begin{restatable}{thm}{sameunrealizable}
\label{thm:unrealize_h_f}
Suppose there exist $h^* \in \mathcal{H}, f^* \in \mathcal{F}, g^* \in \mathcal{G}$ such that  $L_c(f^*, h^*; \mathcal{D}) \le \epsilon_c$ and $L_r(h^*, g^*; \mathcal{D}_X) {\le} \epsilon_r$.
For any $\epsilon_0, \epsilon_1 {\in} (0, 1/2)$, a set $U$ of $m_u$ unlabeled examples and a set $S$ of $m_l$ labeled examples are sufficient to learn to an error $ \epsilon_c {+} \epsilon_1$ with probability $1 {-} \delta$, where
\begin{align}
    \begin{split}
    m_u \ge \frac{C}{\epsilon_0^2} \left[ d(\mathcal{G} \circ \mathcal{H}) \ln\frac{1}{\epsilon_0} {+} \ln\frac{1}{\delta} \right],  
    \end{split}
    \begin{split}
    m_l \ge  \frac{C}{\epsilon_1^2} \left[ d(\mathcal{F} \circ \mathcal{H}_{ \mathcal{D}_X, L_r}(\epsilon_r {+} 2\epsilon_0) )\ln\frac{1}{\epsilon_1} {+} \ln\frac{1}{\delta}  \right] 
    \end{split}
\end{align}
for some absolute constant $C$.
In particular, with probability at least $1 - \delta$, the hypotheses $f \in \mathcal{F}, h \in \mathcal{H}$ that optimize $L_c(f, h; S)$ subject
to $L_r(h; U) \le \epsilon_r + \epsilon_0$ satisfy 
$
L_c(f, h; \mathcal{D}) \le L_c(f^*, h^*; \mathcal{D}) + \epsilon_1.
$
\end{restatable}

\subsection{Different Domains, Unrealizable, Infinite Hypothesis Classes}

In practice, it is often the case that the unlabeled data is from a different domain than the labeled data. For example, state-of-the-art NLP systems are often trained on a large general unlabeled corpus (e.g., the entire Wikipedia) and a small specific labeled corpus (e.g., a set of medical records). That is, the unlabeled data $U$ is from a distribution $\mathcal{U}_X$ different from $\mathcal{D}_X$, the marginal distribution of $x$ in the labeled data. In this setting, we show that our previous analysis still holds.

\diffdomain*
\begin{proof} 
The proof follows that for the setting with the same distribution for input feature vectors in the labeled data and unlabeled data; here we only mention the proof steps involving $\mathcal{U}_X$.  

Even when the unlabeled data is from a different distribution $\mathcal{U}_X$, we still have that with probability $1 - \delta/4$, all $h \in \mathcal{H}$ and $g \in \mathcal{G}$ satisfy $|L_r(h, g; U) - L_r(h, g; \mathcal{U}_X) | \le \epsilon_0$ for the given value of $m_u$. In particular, $L_r(h^*, g^*; U) \le L_r(h^*, g^*; \mathcal{U}_X)  +  \epsilon_0 \le \epsilon_r + \epsilon_0$.
Then the labeled sample size $m_l$ implies that with probability at least $1 - \delta/2$, all hypotheses $h \in \mathcal{H}_{\mathcal{U}_X, L_r}(\epsilon_r + 2\epsilon_0)$ and all $f \in \mathcal{F}$ have $L_c(f, h; S) \le L_c(f, h; \mathcal{D}) + \epsilon_1/2$. Also, for any $h,g$ with $L_r(h, g; U) \le \epsilon_r + \epsilon_0$, we have $L_r(h, g; \mathcal{U}_X) \le \epsilon_r + 2\epsilon_0$, and thus $h \in \mathcal{H}_{\mathcal{D}_X, L_r}(\epsilon_r + 2\epsilon_0)$.
The rest of the proof follows that of Theorem~\ref{thm:covering}.
\end{proof}


\paragraph{Remarks}
We would like to briefly comment on interpreting the reduction in sample complexity of labeled data when using functional regularization in our bounds. The sample complexity bounds are \emph{upper bounds} and are aimed at aiding quantitative analysis by bounding the actual sample size needed for learning (under assumptions on the data and the hypothesis class). However, there exist settings where these bounds are nearly tighter mathematically (e.g., the standard lower bound via VC-dimension). More precisely, there exist hypothesis classes, such that for any learning algorithm, there exists a data distribution and a target function such that a sample, equal in size to the upper bound up to logarithmic factors, is required for learning (a more precise statement can be found in ~\cite{mohri2018foundations}). Additionally, these bounds usually do not take into account the effect of optimization~\cite{zhang2016understanding}.

While these upper bounds are not an exact quantification, they usually align well with the sample size needed for learning in practice, thereby providing useful insights. The reduction in our bounds on using functional regularization can roughly estimate the actual reduction in practice. Further this can provide useful theoretical insights such as the regularization restricting the learning to a subset of the hypothesis class of representation functions. Similar to prior sample complexity studies, we believe our sample complexity bounds can prove to be a useful analysis tool.

\section{Proofs for Applying the Theoretical Framework to Concrete Examples}
\label{app:example}

\subsection{Auto-encoder} \label{app:example_auto}

We first recall the details of the example: $\mathcal{H}$ is the class of linear functions from $\mathbb{R}^d$ to $\mathbb{R}^r$ where ${r < d/2}$, and $\mathcal{F}$ to be the class of linear functions over some activations. That is,
\begin{align}
z = h_W(x) = W x,~~y = f_a(z) = \sum_{i=1}^r a_i \sigma(z_i)\ , \textrm{~where~} W {\in} \mathbb{R}^{r \times d}, ~~ a {\in} \mathbb{R}^{r} \tag{\ref{eqn:auto-encoder}}
\end{align}
Here $\sigma(t)$ is an activation function, the rows of $W$ and $a$ have $\ell_2$ norm bounded by 1. We consider the Mean Square Error (MSE) prediction loss, i.e., $L_c(f, h; x) {=} \|y {-} f(h(x))\|^2_2$. 

Also recall that we assume the data distribution having the following property: let the columns of $B \in \mathbb{R}^{d\times d}$ be the eigenvectors of $\Sigma {:=}\mathbb{E}[xx^\top]$, then the labels are largely determined by the signal in the first $r$ directions: $y {=}  (a^*)^\top z^* {+} \nu$
and $z^* = B_{1:r}^\top x$, where $a^*$ is a ground-truth parameter with $\|a^*\|_2 {\le} 1$, $B_{1:r}$ is the set of first $r$ eigenvectors of $\Sigma$, and $\nu$ is a small Gaussian noise. We also assume that the $r^{\text{th}}$ and ${r{+}1}^{\text{th}}$ eigenvalues of $\Sigma $ are different so that the corresponding eigenvectors can be distinguished.
Let $\epsilon_r$ denote $\mathbb{E}\|x - B_{1:r} B_{1:r}^\top x  \|_2^2$. 

Finally, we recall that the functional regularization $\mathcal{G}$ we used is given by the class of linear functions from $\mathbb{R}^r$ to $\mathbb{R}^d$, i.e., $\hat{x} {=} g_V(z) {=} Vz$ where $V {\in} \mathcal{R}^{d \times r}$ with orthonormal columns. The regularization loss 
$L_r(h, g; x) {=} \|x {-} g(h(x))\|^2_2$. 

For simplicity of analysis, we assume access to infinite unlabeled data, and set the threshold $\tau = \epsilon_r$. Strictly speaking, we need to allow $L_r(h, g; \mathcal{D}_X) \le \epsilon_r + \epsilon$ for a small $\epsilon>0$ due to finite unlabeled data. A similar but more complex argument holds for that case. Here we assume infinite unlabeled data to simplify the presentation and better illustrate the intuition, since our focus is on quantifying the reduction in labeled data. 

Formally, we calculate the sample complexity bounds in the limit $m_u {\rightarrow} {+}\infty$. Equivalently we consider the learning problem: 
\begin{align}
    \min_{f \in \mathcal{F}, h \in \mathcal{H}} L_c(f, h\ ; S), \textrm{~~s.t.~~} L_r(h\ ; \mathcal{D}_X) \le \epsilon_r. 
\end{align}

Let $\mathcal{N}_{\mathcal{C}}(\epsilon)$ denote the $\epsilon$-covering number of a class $\mathcal{C}$ w.r.t.\ the $\ell_2$ norm (i.e., Euclidean norm for the weight vector $a$, and Frobenius norm for the weight matrices $W$ and $V$). Let $L$ denote the Lipschitz constant of the losses (See Appendix~\ref{app:same_domain_unrealizable}). Without regularization, the standard $\epsilon$-net argument shows that the labeled sample complexity, for an error $ \epsilon \ $ close to the optimal, is $\frac{C}{\epsilon^2} \left[ \ln \mathcal{N}_{  \mathcal{F}}\left(\frac{\epsilon}{4L}\right) + \ln \mathcal{N}_{  \mathcal{H}}\left(\frac{\epsilon}{4L}\right) \right]$ for some absolute constant $C$. 
Applying our framework when using regularization, the sample complexity is bounded by $\frac{C}{\epsilon^2} \left[ \ln \mathcal{N}_{\mathcal{F}}\left(\frac{\epsilon}{4L}\right) + \ln \mathcal{N}_{\mathcal{H}_{\mathcal{D}_X, L_r}(\epsilon_r)}\left(\frac{\epsilon}{4L}\right) \right]$.
To quantify the reduction in the bound, we show the following lemma.

\begin{lem}
\label{lem:auto-encoder}
For $\epsilon/4L < 1/2$,
\begin{align}
    \mathcal{N}_{ \mathcal{H}}\left(\frac{\epsilon}{4L}\right) \ge  {d - r \choose r} \mathcal{N}_{ \mathcal{H}_{\mathcal{D}_X, L_r}(\epsilon_r) }\left(\frac{\epsilon}{4L}\right).
\end{align}
\end{lem}
\begin{proof}
First, recall that the regularization loss is 
\begin{align}
    L_r(h, g; \mathcal{D}_X) & = \mathbb{E}_x  \|x - g(h(x))\|^2_2 \notag \\
    &= \mathbb{E}_x  \|x - VWx\|^2_2
\end{align}
which is the $r$-rank approximation of the data.
So in the optimal solution, the columns of $V$ and the rows of $W$ should span the subspace of the top $r$ eigenvectors $\Sigma$. More precisely,
\begin{align}
    L_r(h, g; \mathcal{D}_X)
    & = \mathbb{E}_x [x^\top (I {-} VW)^\top (I {-} VW) x] \notag \\
    & = \mathbb{E}_x [\mathrm{trace}(x^\top (I {-} VW)^\top (I {-} VW) x)] \notag
    \\
    & =  \mathbb{E}_x [\mathrm{trace}( (I {-} VW)^\top (I {-} VW) xx^\top)] \notag \\
    & =  \mathrm{trace}( (I {-} VW)^\top (I {-} VW) \Sigma). \notag
    \\
    & = \mathrm{trace}( (I {-} VW) \Sigma).
\end{align} 
Since $V$ and $W$ are orthonormal and have rank $r$, the optimal $VW$ should span the subspace of the top $r$ eigenvectors of $\Sigma$ and the optimal loss is given by $\epsilon_r$.%
\footnote{The optimal product of $V$ and $W$ should span the subspace of the top $r$ eigenvectors of $\Sigma$. But note that there are different pairs of $V$ and $W$ which can achieve the same product.} 
Since the $r$-th and $r+1$-th eigenvalues of $\Sigma$ are different, the optimal $VW$ is unique, and thus we have
\begin{align}
    \mathcal{H}_{\mathcal{D}_X, L_r}(\epsilon_r) & = \{
    OB_{1:r}^\top: O \in \mathbb{R}^{r \times r}, O \textrm{~is orthonormal}\}. \notag
\end{align}

On the other hand, if $B_S$ refers to the sub-matrix of columns in $B$ having indices in $S$, then clearly,
\begin{align}
    \mathcal{H} & \supseteq \mathcal{H}_S := \{O B_S^\top:O \in \mathbb{R}^{r \times r}, O \textrm{~is orthonormal}\}, \notag
\end{align}
for any $S \subseteq \{r+1, r+2, \ldots, d\}, |S| = r$. By symmetry, $\mathcal{N}_{\mathcal{H}_S}(\epsilon') = \mathcal{N}_{\mathcal{H}_{\mathcal{D}_X, L_r}(\epsilon_r)}(\epsilon')$ for any $\epsilon'>0$. 
Now it is sufficient to prove that $\mathcal{H}_S$ and $\mathcal{H}_{S'}$ are sufficiently far away for different $S$ and $S'$. This is indeed the case, since $\|O B_S^\top - O' B_{S'}^\top\|^2_F > 1$ for any orthonormal $O$ and $O'$:
\begin{align}
    \|O B_S^\top - O' B_{S'}^\top\|^2_F 
    & = \mathrm{trace}\left( (O B_S^\top - O' B_{S'}^\top)^\top (O B_S^\top - O' B_{S'}^\top) \right)
    \\
    & = \mathrm{trace}\left( (O B_S^\top)^\top (O B_S^\top) \right) + \mathrm{trace}\left( (O' B_{S'}^\top)^\top(O' B_{S'}^\top) \right)  \notag
    \\
    & \quad - \mathrm{trace}\left( (O B_S^\top)^\top( O' B_{S'}^\top) \right) - \mathrm{trace}\left(  (O' B_{S'}^\top)^\top (O B_S^\top) \right)
    \\
    & = \|O B_S^\top\|^2_F + \|O' B_{S'}^\top\|^2_F  \notag
    \\
    & \quad - \mathrm{trace}\left( ( O' B_{S'}^\top)(O B_S^\top)^\top \right) - \mathrm{trace}\left(   (O B_S^\top)(O' B_{S'}^\top)^\top \right)
    \\
    & = \|B_S^\top\|^2_F + \|B_{S'}\|^2_F - \mathrm{trace}\left( B_{S'}^\top B_S \right) - \mathrm{trace}\left(B_S^\top B_{S'} \right)
    \\
    & \ge r + r - (r-1) - (r-1) = 2.
\end{align}
This completes the proof. 
\end{proof}


\subsection{Masked Self-supervision} \label{app:example_masked}

We first recall the details of the example: $\mathcal{H}$ is the class of linear functions from $\mathbb{R}^d$ to $\mathbb{R}^r$ where $r < (d-1)/2 $ followed by a quadratic activation, and $\mathcal{F}$ is the class of linear functions from $\mathbb{R}^r$ to $\mathbb{R}$. That is,
\begin{align}
z = h_W(x) {=} [\sigma(w_1^\top x), \ldots, \sigma(w_r^\top x)] \in \mathbb{R}^r \ ,\ y = f_a(z) = a^\top z, \textrm{~where~} w_i\in \mathbb{R}^{d} \ ,  a\in \mathbb{R}^{r}. \tag{\ref{eqn:masked}}
\end{align}
Here $\sigma(t) {=} t^2$ for $t \in \mathbb{R}$ is the quadratic activation function. 
Since $\sigma(ct) {=} c^2 t^2$, by scaling, 
w.l.o.g.\ we can assume that $\|w_i\|_2 {=} 1$ and $\|a\|_2 {\le} 1$. 
Without prior knowledge of data, no regularization refers to end-to-end training on $\mathcal{F} {\circ} \mathcal{H}$. 

Also recall that we assume the data $x$ satisfies $x_1 = \sum_{i=1}^r ((u_i^*)^\top x_{2:d})^2$, where $x_{2:d} = [x_2, x_3, \ldots, x_d]$ and $u_i^*$ is the $i$-th eigenvector of $\Sigma := \mathbb{E}[x_{2:d} x_{2:d}^\top]$. Furthermore, the label $y$ is given by $y = \sum_{i=1}^r a^*_i \sigma((u_i^*)^\top x_{2:d}) + \nu$ for some $\|a^*\|_2 {\le} 1$ and a small Gaussian noise $\nu$.
We also assume a significant difference in the $r^{\text{th}}$ and ${r{+}1}^{\text{th}}$ eigenvalues of $\Sigma$.

Finally, we recall that we used a masked self-supervision  functional regularization by constraining the first coordinate of $w_i$ to be $0$ for $h$, and choosing the regularization function $g(z) {=} \sum_{i=1}^r z_i$ and the regularization loss $L_r(h, g; x) {=} (x_1 {-} g(h_W(x)))^2$. Note that there is only one $g \in \mathcal{G}$, which is a special case of our framework and our sample complexity theorems still apply.

Again for simplicity, we assume access to infinite unlabeled data, and set the threshold $\tau=0$. Our framework shows that the functional regularization via masked self-supervision reduces the labeled sample bound by $\frac{C}{\epsilon^2} \left[ \ln \mathcal{N}_{ \mathcal{H}}\left(\frac{\epsilon}{4L}\right) {-} \ln \mathcal{N}_{  \mathcal{H}_{\mathcal{D}_X, L_r}(0)}\left(\frac{\epsilon}{4L} \right)\right]$ for some absolute constant $C$.  
The following lemma then gives an estimation of this reduction.

\begin{lem}
\label{lem:masked}
For $\epsilon/4L < 1/2$,
\begin{align}
    \mathcal{N}_{ \mathcal{H}}\left(\frac{\epsilon}{4L}\right) \ge  {d - 1 - r \choose r} \mathcal{N}_{ \mathcal{H}_{\mathcal{D}_X, L_r}(0) }\left(\frac{\epsilon}{4L}\right).
\end{align}
\end{lem}
\begin{proof}
By definition,
\begin{align}\label{eqn:masked_lr_bound}
    \mathbb{E} \left[L_r(h, g; x)\right]  &= \ \mathbb{E} \left[  \sum_{i=1}^r u_i^\top x_{2:d} x_{2:d}^\top u_i - \sum_{i=1}^r (u^*_i)^\top x_{2:d} x_{2:d}^\top u^*_i \right]^2 \\
    & \ge  \left(\mathbb{E} |\sum_{i=1}^r u_i^\top x_{2:d} x_{2:d}^\top u_i - \sum_{i=1}^r (u^*_i)^\top x_{2:d} x_{2:d}^\top u^*_i|  \right)^2\\
    & \ge \ \left | \mathbb{E}\sum_{i=1}^r u_i^\top x_{2:d} x_{2:d}^\top u_i - \mathbb{E} \sum_{i=1}^r (u^*_i)^\top x_{2:d} x_{2:d}^\top u^*_i  \right|^2,\\
    & \ge \ \left | \sum_{i=1}^r u_i^\top \Sigma u_i -  \sum_{i=1}^r (u^*_i)^\top \Sigma u^*_i  \right|^2.
\end{align} 
Therefore, $\mathbb{E} \left[L_r(h, g; x)\right] = 0 $ if and only if $u_1{, \ldots,} u_r$ span the same subspace as $u_1^*, \ldots, u_r^*$, i.e., 
\begin{align*}
    \mathcal{H}_{\mathcal{D}_X, L_r}(0) {=} \{
    h_W(x): w_i = [0,u_i], [u_1{, \ldots,} u_r]^\top {=} O [u_1^*, \ldots, u_r^*]^\top, O{\in}\mathbb{R}^{r \times r}, O \textrm{~is orthonormal}\}.
\end{align*}
On the other hand, if $u^*_1, u^*_2, \ldots, u^*_{d-1}$ are the eigenvectors of $\Sigma$, and $U^*_I := [u_{i_1}^*, \ldots, u_{i_r}^*]^\top$ for indices $I = \{i_1, i_2, \ldots, i_r\} \subseteq \{r+1, r+2, \ldots, d-1\}$,  then clearly
\begin{align*}
    \mathcal{H} \subseteq \mathcal{H}_I := \{
    h_W(x): w_i = [0,u_i], [u_1{, \ldots,} u_r]^\top {=} O (U^*_I)^\top, O{\in}\mathbb{R}^{r \times r}, O \textrm{~is orthonormal}\}
\end{align*}
for any $I = \{i_1, i_2, \ldots, i_r\} \subseteq \{r+1, r+2, \ldots, d-1\}$.
By symmetry, $\mathcal{N}_{\mathcal{H}_I}(\epsilon') = \mathcal{N}_{\mathcal{H}_{\mathcal{D}_X, L_r}(0)}(\epsilon')$ for any $\epsilon'>0$.
Using an argument similar to Section~\ref{app:example_auto}, we can show that for two different index sets $I$ and $I'$, any hypothesis in $\mathcal{H}_I$ and any hypothesis in $\mathcal{H}_{I'}$ cannot be covered by the same ball in any $\epsilon'$-cover with $\epsilon'<1/2$. This completes the proof.
\end{proof}
 
\section{Experiments on Concrete Functional Regularization Examples} \label{app:experiments}

\subsection{Auto-Encoder}
\label{app:auto_encoder_experiments}

\myparagraph{Data: } 
We first generate $d$ orthonormal vectors($\{u_i\}_{i=1}^{i=d}$) in $\mathbb{R}^{d}$. We then randomly generate means $\mu_i$ and variances $\sigma_i$ corresponding to each principal component $i\in[1,d]$ such that $\sigma_1 {> \dots >} \sigma_r \gg \sigma_{r+1} {> \dots >} \sigma_d$. The $\mu_i$'s are randomly generated integers in $[0,20]$ and the variances $\sigma_i \ , i \in [1,r]$ are each generated randomly from $[1,10]$ and $\sigma_i \ , i \in [r+1,d]$ are each generated randomly from $[0,0.1]$. We also generate a vector $a\in \mathbb{R}^{r}$ randomly such that $||a||_2 \le 1$. To generate a data point $(x,y)$, we sample $\alpha_i {\sim} \mathcal{N}(\mu_i,\sigma_i) \ \forall i\in[1,d]$ and set $x=\sum_{i=1}^{d} \alpha_i u_i$ and $y=\sum_{i=1}^{r} a_i {\alpha_i}^2 + \nu$ where $\nu {\sim} \mathcal{N}(0,{10}^{-2})$. We use an unlabeled dataset of $10^4$ points (when using the auto-encoder functional regularization), a labeled training set of $10^4$ points and a labeled test set of $10^3$ points. 

\myparagraph{Models: }
$h_W$ corresponds to a fully connected NN, without any activation function, to transform $x \in \mathbb{R}^{d}$ to its representation $h(s) \in \mathbb{R}^{r}$. For prediction on the target task, we use a linear classifier after a quadratic activation on $h(x)$ to obtain a scalar output $\hat{y}$. For functional regularization $g_V$, we use a fully connected NN to transform the representation $h(s) \in \mathbb{R}^{r}$ to reconstruct the input back $\hat{x} \in \mathbb{R}^d$. Our example additionally constrains $V, W$ to be orthonormal. For achieving this, we add an orthonormal regularization~\cite{DBLP:conf/iclr/BrockLRW17,10.5555/3327144.3327339} penalty for each $V, W$ weighted by hyper-parameters $\lambda_1$ and $\lambda_2$ respectively during the auto-encoder reconstruction. For a matrix $M \in R^{a\times b}$, the orthonormal regularization penalty to ensure that the rows of $M$ are orthonormal, is given by $ \sum_{ij} | (M M^\top)_{ij} - I_{ij}|$ where $\sum_{ij}$ is summing over all the matrix elements, $I$ is the identity matrix in $R^{a\times a}$.

\myparagraph{Training Details: }
For end-to-end training, we train the predictor and $h$ jointly using a MSE loss between the predicted target $\hat{y}$ and the true $y$ on the labeled training data set. For functional regularization, we first train $h$ and $g$ using the MSE loss between the reconstructed input $\hat{x}$ and the original input $x$ over the unlabeled data set. Here, we also add the orthonormal regularization penalties. We tune the weights $\lambda_1$ and $\lambda_2$ using grid search in $[10^{-3},10^{3}]$ in multiplicative steps of $10$ to get the best reconstruction (least MSE) on the training data inputs. Now using $h$ initialized from the auto-encoder, we use the labeled training data set to jointly learn the predictor and $h$ using a MSE loss between the predicted target $\hat{y}$ and the true $y$. We report the MSE on the test set as the metric. For all optimization steps we use an SGD optimizer with momentum set to 0.9 where the learning rate is tuned using grid search in $[10^{-5},10^{-1}]$ in multiplicative steps of $10$. We set the data dimension $d=100$ and report the test MSE averaged over 10 runs. 

\myparagraph{t-SNE Plots of Functional Approximations}
To get a functional approximation from a model, we compute and concatenate the output predictions $\hat{y}$ from the model over the test data set. For every model, we obtain a $\mathbb{R}^{1000}$ vector corresponding to the size of the test set. We perform 1000 independent runs for each model (with and without functional regularization) obtaining $2,000$ functional approximation vectors in $\mathbb{R}^{1000}$. We visualise these vectors in 2D using the t-SNE~\cite{vanDerMaaten2008} algorithm.

\myparagraph{Varying Dimension $d$:} We plot the reduction in test MSE between end-to-end training and using functional regularization on varying $d$ in Figure~\ref{fig:d_auto}. Here we fix $r=30$ and vary the data dimension $d$ and present the test MSE scores normalized with the average norm $\|x\|_2^2$ over the test data. As per indications from our derived bounds, the reduction remains more or less constant on varying $d$.

\begin{figure}[t]
\centering
\subfigure[Auto-Encoder]{\label{fig:7}\includegraphics[width=.4\linewidth]{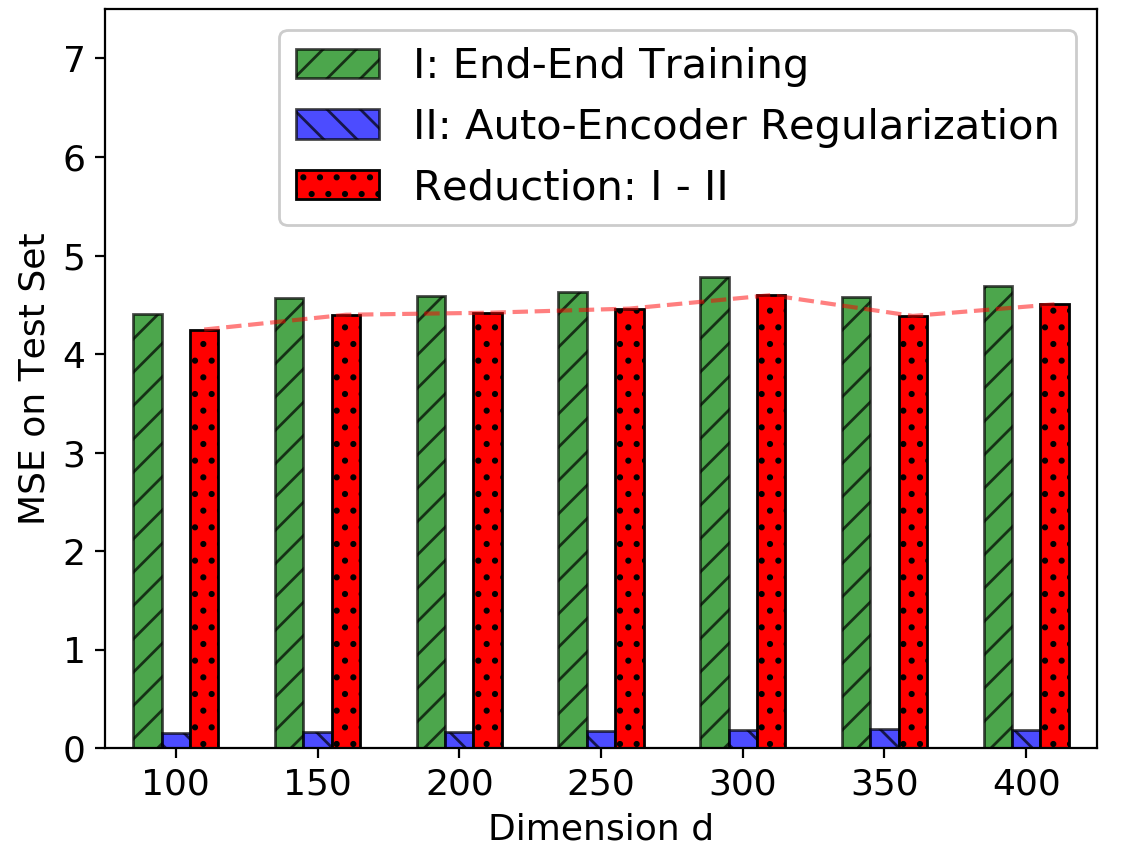}} 
\subfigure[Masked Self-Supervision]{\label{fig:8}\includegraphics[width=.4\linewidth]{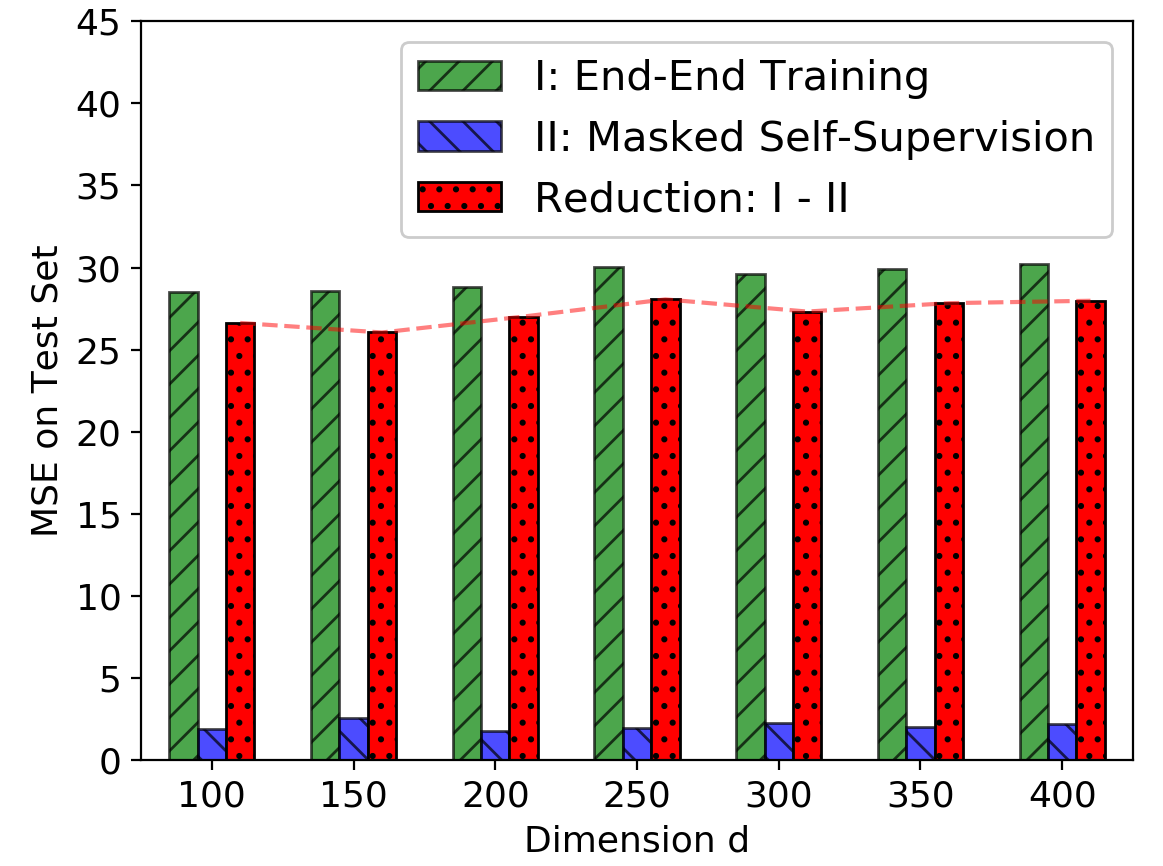}} 
\vspace{-1em}
\caption{Reduction in Test MSE (on using functional regularization with respect to end-to-end training) with dimension $d$. Here $r=30$ and the Test MSE are normalized by the average test $\|x\|_2^2$.}\label{fig:d_auto}\label{fig:d_mask}
\end{figure} 

\subsection{Masked Self-Supervision}
\label{app:masked_experiments}

\myparagraph{Data:}
We first generate $d{-}1$ orthonormal vectors($\{u_i\}_{i=2}^{i=d}$) in $\mathbb{R}^{d-1}$. We then randomly generate means $\mu_i$ and variances $\sigma_i$ corresponding to each principal component $i\in[2,d]$ such that $\sigma_2 {> \dots >} \sigma_{r+1} \gg \sigma_{r+2} {> \dots >} \sigma_{d}$. The $\mu_i$'s are randomly generated integers in $[0,20]$ and the variances $\sigma_i \ , i \in [2,r+1]$ are each generated randomly from $[1,10]$ and $\sigma_i \ , i \in [r+2,d]$ are each generated randomly from $[0,0.1]$. We also generate a vector $a\in \mathbb{R}^{r}$ randomly such that $||a||_2 \le 1$.
To generate a data point $(x,y)$, we sample $\alpha_i {\sim} \mathcal{N}(\mu_i,\sigma_i) \ \forall i\in[2,d]$ and set $x_1=\sum_{i=2}^{r+1} \alpha_i^{2}$, $x_{2:d}=\sum_{i=2}^{d} \alpha_i u_i$ and $y=\sum_{i=2}^{r+1} a_i \alpha_i^{2} + \nu$ where $\nu {\sim} \mathcal{N}(0,{10}^{-2})$. We use an unlabeled dataset of $10^4$ points, a labeled training set of $10^4$ points and a labeled test set of $10^3$ points.

\myparagraph{Models:}
$h_W$ corresponds to a fully connected NN, using a quadratic activation function, to transform $x \in \mathbb{R}^{d}$ to its representation $h(s) \in \mathbb{R}^{r}$. For prediction on the target task we use a linear classifier to obtain the output $\hat{y}$ from the representation $h(x)$. For functional regularization, we sum the elements of the representation $h(x) \in \mathbb{R}^{r}$ to reconstruct the first input $\hat{x_1} \in \mathbb{R}$ back.

\myparagraph{Training Details:}
For functional regularization, we mask the first dimension of the unlabeled data by setting it to $0$ and train $h$ using the MSE loss between the reconstructed $\hat{x_1}$ and the original input dimension $x_1$. Other experimental details remain similar to Section~\ref{app:auto_encoder_experiments}.

Experimental details for the t-SNE plots of functional approximation remain similar to Section~\ref{app:auto_encoder_experiments}.

\myparagraph{Varying Dimension $d$:} Following similar details to Section~\ref{app:auto_encoder_experiments}, we present the graph in Figure~\ref{fig:d_mask}. We can observe that the reduction does not change much on varying $d$ here as well.

\section{Additional Experiments on Functional Regularization}
\label{app:additional_expts}

There have been several empirical studies verifying the benefits of functional regularization across different applications. Here we present empirical results showing the benefits of using functional regularization on a computer vision and natural language processing application.

\subsection{Image Classification}
We consider the application of image classification using the Fashion MNIST dataset~\cite{DBLP:journals/corr/abs-1708-07747} which contains $28 \times 28$ gray-scale images of fashion products from 10 categories. This dataset has 60k images for training and 10k images for testing. We consider a denoising auto-encoder functional regularization using unlabeled data and evaluate its benefits to supervised classification using labeled data. 

\myparagraph{Experimental Details} We use a denoising auto-encoder as the functional regularization when learning from unlabeled data. The encoder consists of three fully connected layers with ReLU activations to obtain the input representation $h(x)$ of 1024 dimensions from an input $x$. The decoder consists of three fully connected layers with ReLU activations to reconstruct the $28 \times 28$ image $\hat{x}$ back from the 1024 dimensional representation $h(x)$. For training, the pixel values of $x$ are normalized to $[0,1]$ and independently corrupted by adding a Gaussian noise with mean $0$ and standard deviation $0.2$. The MSE loss between the $x$ and $\hat{x}$ is used as the regularization loss $L_r$. Training is performed using the Adam optimizer with a learning rate of $3\times 10^{-4}$. For classification, we use a simple linear layer which maps $h(x)$ to the class label $\hat{y}$. The classifier and the encoder are trained jointly using the cross entropy loss between $\hat{y}$ and the original label $y$. We compare the test set accuracy of 1) directly training the encoder and the target classifier using the labeled training data, and 2) pre-training the encoder using the de-noising auto-encoder functional regularization and then fine-tuning its weights along with the target classifier using the labeled training data. We vary the size of the labeled training data and plot the test accuracy averaged across 5 runs in Figure~\ref{fig:fashionmnist-1}.

To visualize the impact of the denoising auto-encoder functional regularization, we follow the details in Appendix~\ref{app:auto_encoder_experiments} to get the functional approximation of the model. For each model, we obtain a $\mathbb{R}^{10000 \times 10}$ matrix with softmax values for 10 target classes for each of the 10000 test points. We perform 100 independent runs for each method (with and without the functional regularization) obtaining $200$ functional approximation vectors in $\mathbb{R}^{100,000}$. We visualise these vectors in 2D using the Isomap~\cite{Tenenbaum2319} algorithm~\footnote{The t-SNE algorithm focuses more on neighbour distances by allowing large variance in large distances, while Isomap approximates geodesic distance via shortest paths thereby working well in practice with larger distances. Compared to the controlled data experiments where the functional approximation lies in $\mathbb{R}^{1000}$, the functional approximation for Fashion-MNIST lies in $\mathbb{R}^{100,000}$, thereby visualizing better via Isomap than t-SNE.} in Figure~\ref{fig:fashionmnist-2}.

\begin{figure}[t]
\centering
\subfigure[]{\label{fig:fashionmnist-1}\includegraphics[width=.38\linewidth]{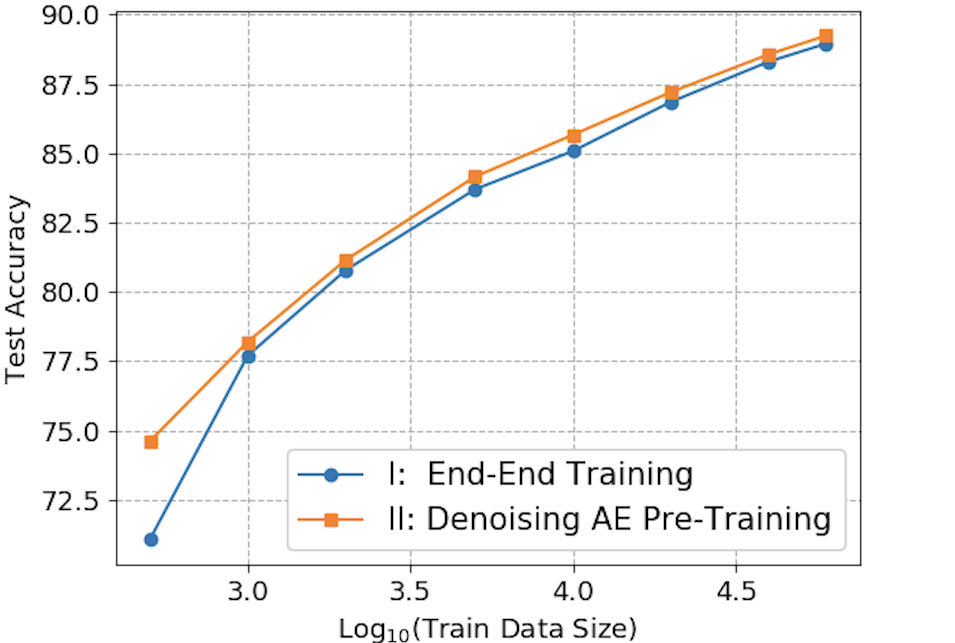}} 
\subfigure[]{\label{fig:fashionmnist-2}\includegraphics[width=.42\linewidth]{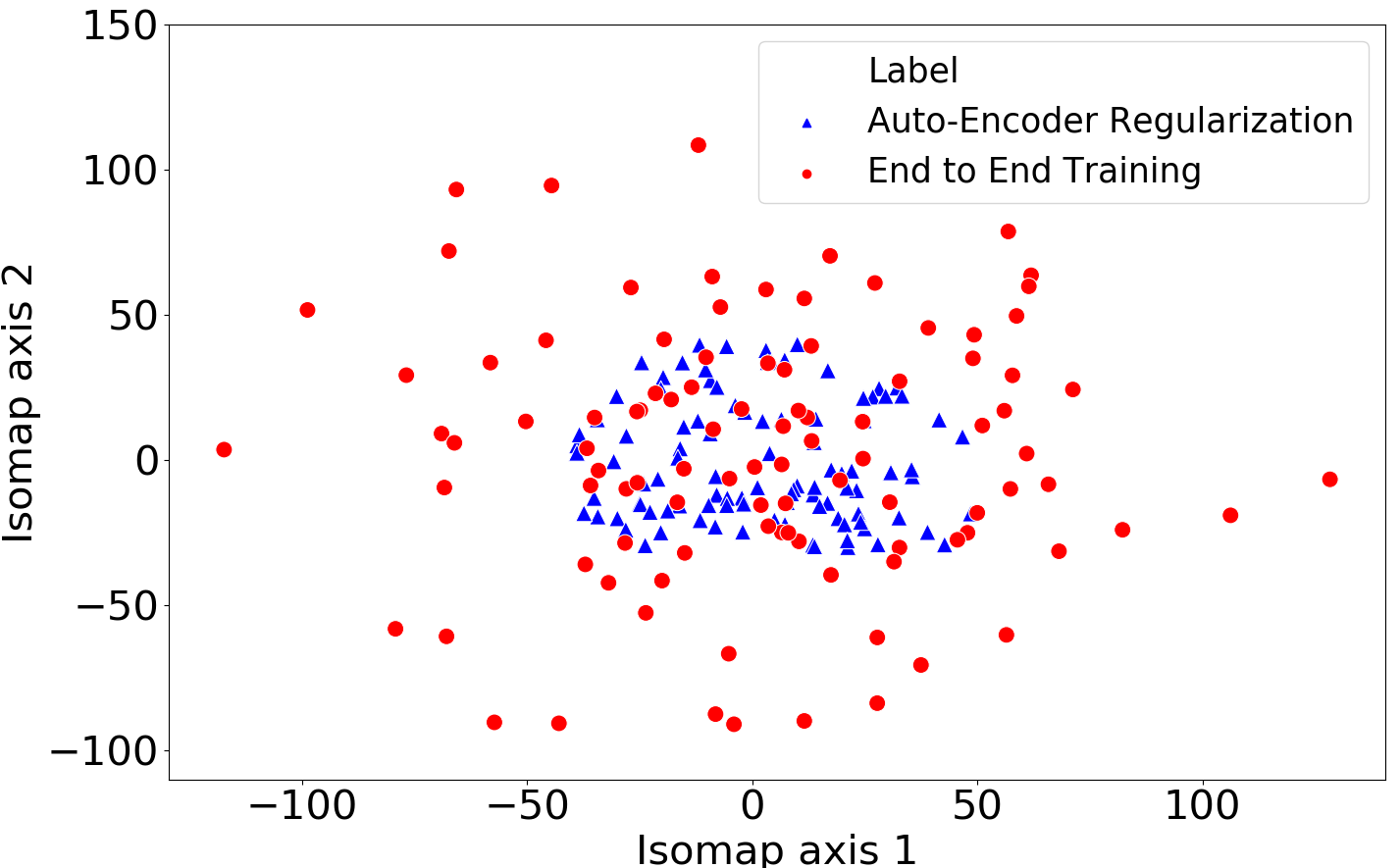}} 
\vspace{-1em}
\caption{Experimental results on \textbf{Fashion-MNIST}. (a) Test  accuracy using de-noising auto-encoder functional regularization compared to end-to-end training on varying the size of labeled training data. (b) The 2D visualization of the functional approximation of 100 independent runs for each method.}
\label{fig:fashionmnist}
\end{figure} 

\myparagraph{Results} From Figure~\ref{fig:fashionmnist-1}, we observe that the test accuracy of end to end training is inferior to that of using functional regularization with unlabeled data across a variety of labeled data sizes. We observe that the difference in the test accuracy between the two methods is highest when the amount of labeled data available is small and the performance gap decreases as the amount of labeled data increases, as predicted by our theory.

Figure~\ref{fig:fashionmnist-2} visualizes the functional approximation learned by the model. It shows that when using the denoising auto-encoder functional regularization, the learned functions stay in a smaller functional space, while they are scattered when using end to end training. This is in line with our empirical observations on controlled data, and our intuition for the theoretical analysis: pruning the representation hypothesis space via functional regularization translates to a compact functional space.

\subsection{Sentence Pair Classification}
We consider the application of sentence pair classification using the Microsoft Research Paraphrase Corpus~\cite{dolan-brockett-2005-automatically}\footnote{\url{https://www.microsoft.com/en-us/download/details.aspx?id=52398}} which has sentence pairs with annotations of whether the two sentences are semantically equivalent. This dataset has approximately 3.7k and 1.7k sentence pairs in the train and test splits respectively. Here we specifically choose the MRPC dataset as it has a smaller size of labeled training data in comparison to most NLP datasets. To show the empirical benefits of using unlabeled data in addition to the limited train data available, we use a pre-trained BERT~\cite{devlin-etal-2019-bert} language model. BERT, based on a transformer architecture, has been pre-trained using a masked token self-supervision task which involves masking a portion of the input sentence and using BERT to predict the masked tokens. This pre-training is done over a large text corpus ($\sim$ 2 billion words) and hence we can view the pre-trained BERT, under our framework, as having already pruned a large fraction of the hypothesis space of $\mathcal{H}$ for learning the representation on the input text.

\myparagraph{Experimental Details}
We compare the performance of fine-tuning the pre-trained BERT with training a randomly initialised BERT from scratch. For the latter, we use three different loss formulations to further study the benefits of regularization on the text representation being learnt: (i) the Cross-Entropy loss $\mathcal{L}_{CE}$ on the predicted output (ii) $\mathcal{L}_{CE}$ along with a $\ell_1$ norm penalty on the representation (i.e, the 768-dimensional representation from BERT corresponding to the [CLS] token) (iii) $\mathcal{L}_{CE}$ along with a $\ell_2$ norm penalty on the representation. We refer to these three different loss formulations as Random, Random-$\ell_1$ and Random-$\ell_2$ respectively for notational simplicity. We want to study how varying the labeled data can impact the performance of different training methods. We present the results in Table~\ref{table:bert}. We use the 12-layer BERT Base uncased model for our experiments with an Adam optimizer having a learning rate $2e^{-5}$. We perform end to end training on the training data and tune the number of fine-tuning epochs. We report the accuracy and F1 scores as the metric on the test data averaged over 3 runs. When randomly initialising the weights of BERT, we use a standard normal distribution with mean 0 and standard deviation of 0.02 for the layer weights and set all the biases to zero vectors. We set the layer norms to have weights as a vector of ones with a zero vector as the bias. When adding the $l_p$ penalty on the BERT representations on randomly initialising the weights, we choose an appropriate weighting function $\lambda$ to make the training loss a sum of the cross entropy classification loss and $\lambda$ times the $l_p$ norm of the BERT representation. The $\lambda$ is chosen $\in [{10}^{-3}, {10}^{3}]$ by validation over a set of 300 data points randomly sampled from the training split. We use the huggingface transformers repository \footnote{\url{https://github.com/huggingface/transformers}} for our experiments.

\myparagraph{Results}
From the table, we observe that the performance of training BERT from pre-trained weights is better than the performance of training the BERT architecture from randomly initialised weights. When viewed under our framework, this empirically shows the benefits of using a learnable regularization function over fixed functions like the $\ell_1$ or $\ell_2$ norms of the representation.

\begin{table}[t]
\centering
\resizebox{0.9\linewidth}{!}{
\begin{tabular}{cccccc}
 \toprule  Train Data Size                   & 200       & 500        & 1000      & 2000      & 3668      \\ \midrule 
BERT-FT                       & \textbf{68.1 / 80.6} & \textbf{71.0 / 80.6}  & \textbf{72.7 / 81.8} & \textbf{74.9 / 82.4} & \textbf{80.3 / 85.7} \\
Random              & 64.1 / 74.8 & 64.7 / 75.66 & 67.0 / 80.1 & 68.9 / 79.0   & 68.9 / 79.3 \\
Random-$\ell_1$  & 54.7 / 65.1 & 62.6 / 75.5  & 63.6 / 76.7 & 63.4 / 76.6 & 66.3 / 79.6 \\
Random-$\ell_2$ & 65.3 / 78.6 & 66.4 / 79.7  & 65.3 / 78.6 & 65.0 / 78.4 & 66.5 / 79.9 \\ \bottomrule
\end{tabular}
}
\vspace{1em}
\caption{Performance of fine-tuning pre-trained BERT (BERT-FT) and end-to-end training of a randomly initialised BERT on varying the \textbf{MRPC} training dataset size. 
Metrics are reported in the format Accuracy/F1 scores on the test dataset. The training data size is 3668 sentence pairs.}
\vspace{-1em}
\label{table:bert}
\end{table}

On increasing the training data size, we observe that the performance of all the four training modes increases. However, we can see that the performance improvement of Random, Random-$\ell_1$ and Random-$\ell_2$ is marginal when compared to the improvement in BERT Fine-tuning. The latter can be attributed to the fact that the pre-trained weights of BERT are adjusted by specialising them towards the target data domain. To support this, in addition to Table~\ref{table:bert}, we also experimented by keeping the BERT weights fixed and only training the classifier. We observe that under such a setting, when we use a small training set, the model is unable to converge to a model different from the initialisation as similarly observed by \cite{kovaleva-etal-2019-revealing}. This means that the learning indeed needs searching over a set of suitable hypotheses.
Thus, we can conclude that unlabeled data helps in restricting the search space, and a small labeled data set can find a hypothesis suitable for the target domain data within the restricted search space, consistent with our analysis.
\end{document}